\documentclass{article}

\usepackage{my_neurips_2023}

\newcommand{\myparagraph}[1]{\paragraph{#1}}

\title{Strategic Data Sharing between Competitors}

\author{%
  Nikita Tsoy \\
  INSAIT, Sofia University\\
  Sofia, Bulgaria\\
  \texttt{nikita.tsoy@insait.ai} \\
  \And
  Nikola Konstantinov\\
  INSAIT, Sofia University\\
  Sofia, Bulgaria\\
  \texttt{nikola.konstantinov@insait.ai} \\
}

\begin{document}

\maketitle

\begin{abstract}
Collaborative learning techniques have significantly advanced in recent years, enabling
  private model training across multiple organizations. Despite this
  opportunity, firms face a dilemma when considering data sharing with
  competitors---while collaboration can improve a company's machine learning
  model, it may also benefit competitors and hence reduce profits. In this
  work, we introduce a general framework for analyzing this data-sharing
  trade-off. The framework consists of three components, representing the
  firms' production decisions, the effect of additional data on model quality,
  and the data-sharing negotiation process, respectively. We then study an
  instantiation of the framework, based on a conventional market model from
  economic theory, to identify key factors that affect collaboration
  incentives. Our findings indicate a profound impact of market
conditions on the data-sharing incentives. In particular, we find that reduced competition, in terms of the
  similarities between the firms' products, and harder learning tasks
  foster collaboration.
\end{abstract}

\section{Introduction}
\label{sec:intro}
Machine learning has become integral to numerous business functions in the past
decade, such as service operations optimization, the creation of new products,
consumer service analytics, and risk analysis \citep{survey22}. Despite the
transformative power of machine learning, its efficacy hinges significantly on
the quality and quantity of the training data, making data a key asset
corporations can use to increase their profit.

One way to enhance data access and machine learning models is collaboration via
data sharing \citep{r20, d22}. Data-sharing schemes can bring
benefits in many industries, such as agriculture, finance, and healthcare
\citep{d22, w19, r20}. However, at least two barriers exist to such
collaborations. The first is privacy and the associated regulatory concerns,
which can be partially addressed by new collaborative learning techniques, such
as federated learning \citep{k21}. The second is proper incentives for
collaboration. If the entities have no such incentives, they may not
collaborate at all, free-ride \citep{b21, k22}, or attack the shared model
\citep{blanchard2017machine}.

Such conflicting incentives arise naturally between market competitors. On the
one hand, the competitors are an appealing source of data, since they operate
on the same market and have data of a similar type. On the other hand, the
collaboration could also strengthen the competitors' machine learning models,
potentially leading to downstream losses for the company. This concern is
especially important for big firms that can influence prices and act
strategically by considering how their decisions affect competitors' responses.
Strategic actions might increase profits through various channels and produce
complex downstream effects. For example, firms might collude to capture more
revenue\footnote{\scriptsize{\url{https://www.nytimes.com/2021/01/17/technology/google-facebook-ad-deal-antitrust.html}}}
or engage in a price war to win a
market.\footnote{\scriptsize{\url{https://www.reuters.com/article/us-uber-results-breakingviews-idUSKCN1UY2X5}}}
Thus, \emph{data sharing between big firms might have a complicated downstream
effect on their profits}.

\myparagraph{Our contributions} Despite the significance of this data-sharing
trade-off, particularly for large, tech-savvy companies, there is limited
research on how market competition affects collaborative learning incentives.
Our work aims to fill this gap by proposing a framework to investigate
data-sharing incentives. The framework is modular and consists of three
components: market model, data impact, and collaboration scheme. These
components represent the firms' production decisions, the effect of additional
data on model quality and the data-sharing negotiation process, respectively.

To investigate the key factors that influence data-sharing decisions, we
instantiate our framework using a conventional market model from economic
theory and a data impact model grounded in learning theory. We examine
three potential collaboration schemes: binary collaboration between two firms,
partial collaboration between two firms, and binary collaboration among
multiple firms. Our findings consistently indicate a profound impact of market
conditions on the collaboration incentives. In the case of two firms, we
theoretically demonstrate that collaboration becomes more appealing as market
competition, measured by the similarities of the firms' products, decreases or
the learning task complexity increases. For multiple firms, we conduct
simulation studies that reveal similar trends.

\section{Related work}
\label{sec:rel_work}
\myparagraph{Data sharing incentives in machine learning}
Incentives for collaboration in machine learning constitute an important
research topic, especially in the context of new learning paradigms, such as
federated learning \citep{k21}. Some notable lines of work concern incentives
under data heterogeneity \citep{d21, w22}; compensation for computation,
energy, and other costs related to training \citep{yu20, tu22, liu22}; fair
distribution of the benefits resulting from collaboration \citep{lyu20b, lyu20,
b21}; and free-riding \citep{richardson20, k22}. We refer to \citet{z21, y20,
zhan21} for a detailed overview.

A fundamental difference between our work and the research described above is
that we explicitly consider the downstream market
incentives of the firms and their effects on data-sharing decisions. These incentives are orthogonal to those previously
studied in the literature. For instance, the data-collection game of
\citet{k22} can be included in our framework as an additional stage of the game
after our coalition formation stage. In addition, our framework is not specific
to federated learning, as it can be applied regardless of the learning procedure used on the
shared dataset.

To our awareness, \citet{wu22} is the only work considering data sharing
between market competitors. However, the authors do not explain the mechanisms
behind the effects of ML models' quality on the market. Therefore, their
framework can not predict whether a potential coalition will be beneficial for
the company but only indicate the benefits of the coalition post factum. In addition,
they use the marketing model of \citet{r93} instead of the classic economic
models \citep{t88} to model competition, potentially constraining the
applicability. Finally, their notion of the sustainability of coalitions does
not arise from standard concepts in game theory. In contrast, we base our
analysis on the standard Nash equilibrium concept \citep{n51}.

\myparagraph{Market competition and information sharing} Competitive behavior is a
well-studied topic in the economic literature \citep{t88}. We refer to
\citet{s89} for an extensive review of the theory of competition. From a more
applied perspective, \citet{r07} discuss how to test the behavior of firms
empirically, while \citet{b19} overview the modern empirical studies of
industrial organization. A rich economic literature
studies the incentives of competing firms to share information \citep{r96,
be19}. For example, demand \citep{v84} or costs \citep{f84} might be
unobservable, and firms might share information about them to improve their
decisions. In contrast to these studies, in our model, firms share information
to improve their products or optimize their costs, not to reduce uncertainty.
Thus, the mechanisms and conclusions of these earlier works and our work
differ.

\section{Data-sharing problem}
\label{sec:setting}
In this section, we provide a general formulation of the data-sharing problem,
consisting of three components: market model, data impact, and collaboration
scheme. The market model outlines
consumers' and firms' consumption and production decisions. Data impact explains how additional
data affects machine learning model quality. Lastly, the collaboration scheme details the data-sharing negotiation process among the
competitors. Given a specific application, these three components can be derived/modeled by market research or sales teams, operation management and data scientists, and mediators, respectively.

We begin with an illustrative example that will help us to clarify the abstract
concepts used in our framework and then introduce the three framework components in
their full generality.

\subsection{Running example}

Consider a city with a taxi market dominated by a few firms. Each firm collects
data (e.g., demand for taxis and traffic situation) to train a machine learning
model for optimizing driver scheduling and other internal processes, which is
crucial for improving scheduling quality, reducing costs, or enhancing
services. Each company can use only its own data or take part in a
collaborative training procedure, like federated learning, with its
competitors. Collaboration can substantially improve its machine learning
model, but it may also strengthen the models of its competitors. Thus, the company
must carefully evaluate the impact of data-sharing on its downstream profits.

\subsection{Market model}
\label{sec:general_market_model}

The market model encompasses consumer actions (demand factors) and firm
production actions (supply factors). We consider a market with $m$ firms, $F_1,
\dots, F_m$, each producing $q_i \in \mathbb{R}_+$ units of good $G_i$ and
offering them to consumers at price $p_i \in \mathbb{R}_+$, where $\mathbb{R}_+
= [0, \infty)$. In our example, $G_1, \dots, G_m$ represent taxi services from
different companies, with consumers being city travelers, $q_i$ are the total
number of kilometers serviced by company $i$, and $p_i$ are the prices per
kilometer.

In this setting, prices and quantities are $2 m$ unknown variables. Correspondingly, we need
$2 m$ constraints to describe the consumers' and firms' market decisions, which
we derive from the consumers' and the firms' rationality. In contrast to
standard market models, our framework allows product utilities and costs to depend
on the qualities of the machine learning models of the firms $\vec{v} = (v_1,
\ldots, v_m)$. This dependence models the impact of the learned models on the services and production pipelines of the firms.

\myparagraph{Consumers' behavior} Each consumer $j$ optimizes their utility
$u^j(\vec{g}^j, \vec{q}^j, \vec{v})$ by purchasing goods given the market
prices. We assume that consumers cannot influence prices, as there are many
consumers and they do not cooperate.

The consumer's utility depends on three factors. The first one is the consumed
quantities $\vec{q}^j = (q_1^j, \ldots, q_m^j)$ of products $G_1, \ldots, G_m$.
The second is the machine learning models' qualities $\vec{v}$, as these may
impact the utility of the corresponding products. The last one is consumed
quantities $\vec{g}^j = (g_1^j, \ldots, g_k^j)$ of goods outside the considered
market (e.g., consumed food in our taxi example). Assuming that each consumer
$j$ can only spend a certain budget $B^j$ on all goods, one obtains the
following optimization problem
\begin{equation}
  \label{eq:cons_problem}
  \max_{\vec{g}^j, \vec{q}^j} u^j(\vec{g}^j, \vec{q}^j, \vec{v}) \text{ s.t. }
  \sum_{l=1}^k \tilde{p}_l g_l^j + \sum_{i=1}^m p_i q_i^j \le B^j,
\end{equation}
where $\tilde{\vec{p}}$ are the outside products' prices (which we consider
fixed). The solution to this problem $q^{j, *}_i(\vec{p}, \vec{v})$ determines
the aggregate demanded quantity of goods
\begin{equation}
  \label{e:agg-demand}
  q_i(\vec{p}, \vec{v}) \defeq \sum_j q^{j, *}_i(\vec{p}, \vec{v}).
\end{equation}
The functions $q_i(\vec{p}, \vec{v})$ (known as the demand equations) link the
prices $p_i$ and the demanded quantities $q_i$ and constitute the first $m$
restrictions in our setting.

\myparagraph{Firms' behavior} Firms maximize their expected profits, the
difference between revenue and cost,
\begin{equation}
  \label{eq:profit}
  \varPi_i^e = \E_{\vec{v}}(p_i q_i - C_i(q_i, v_i)).
\end{equation}
Here $C_i(q_i, v_i)$ is the cost of producing $q_i$ units of good for the firm
$F_i$, and the expectation is taken over the randomness in the models' quality,
influenced by the dataset and the training procedure. Note that the model quality
is often observed only after testing the model in production (at test time).
Therefore, we assume that the firms reason in expectation instead. In our
running example, $C_i$ depends on driver wages, gasoline prices, and scheduling
quality.


The firms may act by either deciding on their produced quantities, resulting in what is known as the Cournot competition model \citep{c38}; or on their prices, resulting in the Bertrand competition model \citep{b83}. As the demand equations (\ref{e:agg-demand})
may interrelate prices and quantities for various products, firms strategically
consider their competitors' actions, resulting in a Nash equilibrium. The equilibrium conditions provide another $m$ constraints, enabling us to
solve the market model entirely.

\subsection{The impact of data on the market}
\label{sec:data_effect}

Each company $F_i$ possesses a dataset $D_i$ that can be used to train a
machine learning model (e.g., trip data in the taxi example) and may opt to
participate in a data-sharing agreement with some other firms. Denote the final
dataset the company acquired  by $D^c_i$. The company then uses the dataset to
train a machine learning model. We \emph{postulate two natural ways that the
model quality $v_i$ can impact a company}. The first one is by \emph{reducing the company's
production costs} $C_i(q_i, v_i)$, for example, by minimizing time in traffic
jams. The second is by \emph{increasing the utility of its products} $q_i(\vec{p},
\vec{v})$, for example, by minimizing the waiting time for taxi arrival.

\subsection{Collaboration scheme}

Following classic economic logic, we posit that \emph{firms will share data if
this decision increases their expected profits} $\varPi^e_i$. Since firms can
not evaluate the gains from unknown data, we assume that they know about the
dataset characteristics of their competitors (e.g., their number of samples and
distributional information). Although expected profit maximization determines
individual data-sharing incentives, forming a coalition necessitates mutual
agreement. Therefore, the precise game-theoretic formulation of the
data-sharing problem depends on the negotiation process details, such as the
number of participants and full or partial data sharing among firms.

Consider our taxi example. If only two firms are present, it is natural that
they will agree on sharing their full datasets with each other if and only if
they both expect that this will lead to increased profits. Partial data-sharing
will complicate the process, leading to an intricate bargaining process between
the firms. Finally, if many companies are present, the data-sharing decisions
become even more complicated, as a firm needs to reason not only about the
data-sharing benefits but also about the data-sharing decisions of other firms.

\paragraph{Legal and training costs considerations} Other considerations can
also enter into our framework through the collaboration scheme. In particular,
if firms have legal requirements on data usage, they should impose suitable
constraints on possible data-sharing agreements. If the collaborative learning
procedure involves training large models, the firms should negotiate how to
split training costs. For example, they might split the costs equally among all
coalition members or proportionally to the sizes of their datasets.

\section{Example market and data impact models}
\label{sec:costs}
In this section, we instantiate the general framework using a conventional
market model from economic theory and a natural data impact model justified by
learning theory. These models allow us to reason quantitatively about the
data-sharing problem in the following sections, leading to the identification
of several key factors in the data-sharing trade-off. To focus solely on the
data-sharing trade-off, we make several simplifying assumptions: the firms'
data is homogeneous, training costs are negligible, and legal constraints are
not present or are automatically fulfilled.

\subsection{Market model}

In order to make quantitative statements about the firms' actions and profits
within the general model of \cref{sec:general_market_model}, one needs to make
further parametric assumptions on utilities $u^j$ and cost functions $C_i$ and
specify the competition game. To this end, we use a specific utility and cost
model standard in the theoretical industrial organization literature
\citep{t88, c90}. Despite its simplicity, this model effectively captures the
basic factors governing market equilibrium for many problems and is often used to obtain
qualitative insights about them.

\subsubsection{Demand}
We assume that, in the aggregate, the behavior of consumers
(\ref{eq:cons_problem}) can be described by a representative consumer with
quasi-linear quadratic utility (QQUM, \citealt{d79})
\begin{equation}
  \label{eq:qqum}
  \max_{g, \vec{q}} u(g, \vec{q}) \defeq
  \sum_{i = 1}^m q_i - \Par*{\frac{\sum_i q_i^2 + 2 \gamma \sum_{i > j}
  q_i q_j}{2}} + g = \vec{\iota}^{\tran} \vec{q} - \frac{\vec{q}^{\tran}
  \mat{G} \vec{q}}{2} + g \text{ s.t. } g + \vec{p}^{\tran} \vec{q} \le B.
\end{equation}
Here, $\vec{\iota} = (1, \dots, 1)^{\tran}$, $\mat{G} = (1 - \gamma) \mathbf{I}
+ \gamma \vec{\iota} \vec{\iota}^{\tran}$ and $g$ is the quantity from a single
outside good.

Additionally, $\gamma \in \Par*{-\frac{1}{m-1}, 1}$ is a measure of
substitutability between each pair of goods: higher $\gamma$ corresponds to
more similar goods. In our running example, $\gamma$ describes the difference
in service of two taxi companies, such as the difference in the cars' quality
or the location coverage. We refer to \citet{c19} for a detailed discussion on
the QQUM model and the plausibility of assuming an aggregate consumer
behavior.

In this case, the exact form of the demand equations (\ref{e:agg-demand}) is
well-known.
\begin{lemma}[\citealt{a17m}]
  \label{l:demand}
  Assume that $\mat{G}^{-1} (\vec{\iota} - \vec{p}) > 0$ and $\vec{p}^{\tran}
  \mat{G}^{-1} (\vec{\iota} - \vec{p}) \le B$. The solution to the consumer
  problem (\ref{eq:qqum}) is
  \begin{align}
    \label{e:inv-demand}
    p_i = 1 - q_i - \gamma \sum_{j \neq i} q_j \iff
    q_i = \frac{1 - \gamma - p_i - \gamma (m - 2) p_i + \gamma
    \sum_{j \neq i} p_j}{(1 - \gamma) (1 + \gamma (m - 1))}.
  \end{align}
\end{lemma}
The technical conditions $\mat{G}^{-1} (\vec{\iota} - \vec{p}) > 0$ and
$\vec{p}^{\tran} \mat{G}^{-1} (\vec{\iota} - \vec{p}) \le B$ ensure that the
consumers want to buy at least a bit of every product and do not spend all of
their money in the considered market, respectively. For completeness, we prove
\cref{l:demand} in \cref{sec:demand-proof}.

\subsubsection{Supply}
\label{sec:supply}
We assume that the cost functions (\ref{eq:profit}) are linear in the
quantities
\begin{equation}
  \label{e:profit}
  \varPi_i^e = \E(p_i q_i - c_i q_i),
\end{equation}
where the parameter $c_i$ depends on the quality of the machine learning model.
We denote $c_i^e = \E_{D_i^c}(c_i)$.

\begin{remark}
  Here we assume that the quality of the machine learning model only affects the
  production costs, and not the utilities. However, in this market model, it is
  equivalent to assuming that the machine learning model affects the consumers'
  utility via increasing the effective quantities of the produced goods.
  Specifically, assume that each product has an internal quality $w_i$ and it
  increases the effective amount of good in the consumer's utility
  (\ref{eq:qqum}), resulting in utility $u(g, w_1 q_1, \dots, w_m q_m)$. This
  model is equivalent to the model above after substituting $q_i$, $p_i$, and
  $c_i$ with their effective versions $\tilde{q}_i = w_i q_i$, $\tilde{p}_i =
  p_i / w_i$, and $\tilde{c}_i = c_i / w_i$.
\end{remark}

We now derive the remaining constraints on the prices and costs. Note that
these equilibria depend on the expected costs $c^e_i$ and hence on the
quality of the machine learning models.

\myparagraph{Cournot competition} Each firm acts by choosing an output level
$q_i$, which determines the prices (\ref{e:inv-demand}) and expected profits
(\ref{e:profit}). The next standard lemma describes the Nash equilibrium of
this game.

\begin{lemma}
  \label{l:cournot}
  Assume that the expected marginal costs are equal to $c^e_1, \ldots, c^e_m$, and
  companies maximize their profits (\ref{e:profit}) in the Cournot competition
  game. If $\, \forall i \, (2 - \gamma) (1 - c_i^e) > \gamma \sum_{j \neq i}
  (c_i^e - c_j^e)$, Nash equilibrium quantities and profits are equal to
  \[
    q_i^* = \frac{2 - \gamma - (2 + \gamma (m - 2)) c_i^e + \gamma \sum_{j \neq
    i} c_j^e}{(2 - \gamma) (2 + (m - 1) \gamma)}, \ \varPi_i^e = (q_i^*)^2.
  \]
\end{lemma}

\myparagraph{Bertrand competition} Each firm acts by setting the price $p_i$ for
their product, which determines the quantities (\ref{e:inv-demand}) and expected
profits (\ref{e:profit}). The following lemma describes the Nash equilibrium of
this game.

\begin{lemma}
\label{l:bertrand}
  Assume that the expected marginal costs are equal to $c^e_1, \ldots, c^e_m$, and
  companies maximize their profits (\ref{e:profit}) in the Bertrand competition
  game. If $\, \forall i \, d_1 (1 - c_i^e) > d_3 \sum_{j \neq i} (c_i^e -
  c_j^e)$, Nash equilibrium prices and profits are equal to
  \[
    p_i^* = \frac{d_1 + d_2 c_i^e +  d_3 \sum_{j \neq i} c_j^e}{d_4}, \
    \varPi^e_i = \frac{(1 + \gamma (m + 2)) (p_i^* - c_i^e)^2}{(1 -
    \gamma) (1 + \gamma (m - 1))},
  \]
  where $d_1, \ldots, d_4$ depend only on $\gamma$ and $m$.
\end{lemma}

In both lemmas, the technical conditions $\forall i \, (2 - \gamma) (1 -
c_i^e) > \gamma \sum_{j \neq i} (c_i^e - c_j^e)$ and $\forall i \, d_1 (1 -
c_i^e) > d_3 \sum_{j \neq i} (c_i^e - c_j^e)$ ensure that every firm produces
a positive amount of good in the equilibrium. We prove these Lemmas in
Sections \ref{sec:cournot-proof} and \ref{sec:bertrand-proof}.

\subsection{Data impact model}
In order to reason strategically about the impact of sharing data with others,
firms need to understand how additional data impacts their costs. Here, we
consider the case where all datasets are sampled from the same
distribution $\mathcal{D}$. While heterogeneity is an important concern in
collaborative learning, we focus on the homogeneous case since heterogeneity is
orthogonal to the incentives arising from market equilibrium considerations.

Homogeneity allows us to consider the expected costs in form $c^e_i =
c^e(n_i^c)$, where $n_i^c = |D^c_i|$ is the number of points the company $i$
has access to. Using the examples below, we motivate the following functional
form for $c^e(n)$
\begin{equation}
  \label{e:data-dep}
  c^e(n) = a + \frac{b}{n^\beta}, \ \beta \in (0, 1].
\end{equation}
In this setting, $\beta$ indicates the \emph{difficulty of the learning task},
higher $\beta$ corresponds to a simpler task \citep{t04}, $a$ represents the
marginal costs of production given a perfect machine learning model, while $a +
b$ corresponds to the cost of production without machine learning
optimizations. Additionally, we assume that $a < 1$ and $\frac{b}{1-a}$ is
small enough, so that the technical requirements of Lemmas \ref{l:demand},
\ref{l:cournot}, and \ref{l:bertrand} are satisfied (see Equations
(\ref{eq:cournot-restriction}) and (\ref{eq:bertrand-restriction})).
Intuitively, this requirement ensures that firms do not exit the market during
the competition stage.

In the examples below, we consider a setting where a company needs to perform
action $\s \in \mathbb{R}^n$ during production that will impact its costs.
However, there is uncertainty in the production process coming from random
noise $\X$ valued in $\mathbb{R}^m$. Thus, the cost of production is a random
function $c(\s, \X)\colon \mathbb{R}^n \times \mathbb{R}^m \to \mathbb{R}_{+}$.

\myparagraph{Asymptotic normality} Assume that the firms use background
knowledge about their production processes in the form of a structural causal
model of $\X$. However, they do not know the exact parameters of the model and
use data to estimate them. Suppose a firm uses the maximum likelihood estimator
to find the parameters and chooses the optimal action $\s_\text{fin}$ based on
this estimate. In this case, the result of optimization
$\E_{\X}(c(\s_\text{fin}, \X))$ will approximately have a generalized
chi-square distribution \citep{j83}, resulting in the following approximation
\begin{equation}
  \label{e:mle}
  \E_{D, \X}(c(\s_\text{fin}, \X)) \approx a + \frac{b}{n}.
\end{equation}
We refer to \cref{sec:asymptotic_normality} for further details. This result
implies the same dependence of the expected marginal costs on the total number
of samples $n$ as \cref{e:data-dep} with $\beta = 1$.

\myparagraph{Stochastic optimization} A similar dependence arises when the
company uses stochastic optimization to directly optimize the expected cost
$c(\s) = \E_{\X} (c(\s, \X))$. If the company uses an algorithm with provable
generalization guarantees (e.g., SGD with a single pass over the data,
\citealt{bubeck2015convex}) and the function $c(\s)$ is strongly convex, the
outcome $\s_\text{fin}$ will satisfy
\[
  \E_D(c(\s_\text{fin}) - c(\s^*)) = \O \Par*{\frac{1}{n}},
\]
where $\s^*$ is the optimal action, resulting in the same dependence as in
the previous paragraph.

\myparagraph{Statistical learning theory} Another justification for the
dependency on the number of data points comes from statistical learning theory.
The observations from this subsection are inspired by similar arguments in
\citet{k22}. Assume that a firm trains a classifier used in the production
process. In this case, the cost $c$ depends on the classifier accuracy $a$, and the
cost overhead will satisfy
\[
  c(a) - c(a^*) \approx -c'(a^*) (a^* - a),
\]
where $a^*$ is the optimal accuracy achievable by a chosen family of
classifiers. If the family of classifiers has a finite VC dimension $d$, a
well-known statistical bound \citep{shalev2014understanding} gives
\[
  a^* - a = \O \Par*{\sqrt{\frac{d}{n}}},
\]
motivating the same dependence as \cref{e:data-dep} with $\beta = 1 / 2$.

\section{Data sharing between two firms}
\label{sec:data_sharing}
Having introduced example market and data impact models, we can specify a
negotiation scheme and quantitatively analyze the data-sharing problem. Through
this, we aim to obtain qualitative insights into how market parameters impact
data-sharing decisions. Following conventional economic logic \citep{t88, c90},
we start with the simplest possible collaboration scheme, in which two
companies make a binary decision of whether to share all their data with each
other. In the next sections, we proceed to the more complicated situations of
partial data sharing and data sharing between many parties.

\myparagraph{Full data sharing between two firms} According to our framework,
both companies compare their expected profits at the Nash equilibrium (from
Lemma \ref{l:cournot} or \ref{l:bertrand}) for two cases, when they collaborate
and when they do not. Then \emph{they share data with each other if and only if
they both expect an increase in profits from this action}.

\begin{remark}
  Notice that we can incorporate some legal requirements into this scheme by
  redefining full data-sharing action. For example, if sharing consent
  constraints, araising from copyright or other data ownership constraint, are
  present, we could assume that the companies will share only a part of their
  dataset for the collaborative learning. While, previously, the competitor
  gets access to all data of the firm $n^\text{share}_{-i} = n_{-i} + n_i$, now
  they will get access to only data of people who agree with data sharing
  $n^\text{share}_{-i} = n_{-i} + n_{i, \text{consent}}$.
\end{remark}

For Cournot competition, \cref{l:cournot} and \cref{e:data-dep} give the
following collaboration criterion
\[
  \forall i \quad \varPi_\text{share}^e > \varPi_{i, \text{ind}}^e \iff
  (2 - \gamma) (n_i^{-\beta} - (n_1 + n_2)^{-\beta}) > \gamma (n_{-i}^{-\beta}
  - n_i^{-\beta}),
\]
where $n_{-i}$ is the size of the data of the player that is not $i$,
$\varPi_{i, \text{share}}^e$ is the expected profit in collaboration, and
$\varPi_{i, \text{ind}}^e$ is the expected profit when no collaboration occurs.

Similarly, in the case of Berntrand competition, \cref{l:bertrand} gives
\[
  \forall i \quad \varPi_\text{share}^e > \varPi_{i, \text{ind}}^e \iff
  (2 - \gamma - \gamma^2) (n_i^{-\beta} - (n_1 + n_2)^{-\beta}) > \gamma
  (n_{-i}^{-\beta} - n_i^{-\beta}).
\]

The theorem below describes the properties of these criteria (see proof in
\cref{sec:main_proof}).

\begin{theorem}
  \label{t:main}
  In the case of $\gamma \le 0$, it is profitable for the firms to
  collaborate. In the case of $\gamma > 0$, there exists a value
  $x_t(\gamma, \beta)$, where $t \in \{\text{Bertrand}, \text{Cournot}\}$ is
  the type of competition, such that, for the firm $i$, it is profitable to
  collaborate only with a competitor with enough data:
  \[
    \varPi^e_\text{share} > \varPi^e_{i, \text{ind}} \iff \frac{n_{-i}}{n_1 +
    n_2} > x_t(\gamma, \beta).
  \]
  The function $x_t$ has the following properties:
  \begin{enumerate}
    \item $x_t(\gamma, \beta)$ is increasing in $\gamma$.
    \item $x_\text{Bertrand} \ge x_\text{Cournot}$.
    \item $x_t(\gamma, \beta)$ is increasing in $\beta$.
  \end{enumerate}
\end{theorem}

\myparagraph{Discussion}
The theorem above indicates that the \emph{firms are more likely to collaborate
when either the market is less competitive} (properties 1 and 2) \textit{or the
learning task is harder} (property 3). Indeed, the threshold $x$ becomes
smaller when the products are less similar ($\gamma$ is bigger), making the
market less competitive. Similarly, $x$ becomes smaller in the Cournot
setting since it is known to be less competitive than the Bertrand one \citep{s89}.
Finally, when the learning task is harder ($\beta$ is smaller), $x$ decreases,
making collaboration more likely. Intuitively, it happens because the decrease
in cost (\ref{e:data-dep}) from a single additional data point is higher for
smaller $\beta$.

In the supplementary material, we explore several extensions
for the two firms case. In \cref{s:diff-costs} and \cref{s:diff-demand}, we
look at different cost and utility functions, respectively. In
\cref{s:asym-costs}, we consider the case where $a$, $b$ and/or $\beta$ might
differ among firms and derive an analog of \cref{t:main}. In
\cref{s:heterogeneity}, we consider the data-sharing problem with two companies
with heterogeneous data in the context of mean estimation. In
\cref{sec:welfare}, we look at the welfare implications of data sharing.

\section{Partial data sharing}
\label{sec:partial_sharing}
In this section, we study a two-companies model in which companies can share
any fraction of their data with competitors. For brevity, we only study this
model in the case of Cournot competition. In this setup, each firm $F_i$
chooses to share a fraction $\lambda_i \in [0, 1]$ of data with its competitor
and expects its costs to be
\[
  c_i^e = c^e(n_i + \lambda_{-i} n_{-i}) = a + b (n_i + \lambda_{-i}
  n_{-i})^{-\beta},
\]
where $\lambda_{-i}$ is the fraction of data shared by the company that is not
$i$.

\begin{remark}
  Notice that we can incorporate some legal requirements into this scheme by
  constraining the action spaces of participants. For example, one can
  implement sharing consent constraints, araising from copyright or other data
  ownership constraint, by constraining the firms’ choices of $\lambda$. While,
  previously, the firms were able to share all data $\lambda \in [0, 1]$, now
  they can share only data of people who agree with data sharing $\lambda \in
  [0, \lambda_\text{consent}]$.
\end{remark}

We will use the Nash bargaining solution
\citep{b86} to describe the bargaining process between the firms. Namely, the
firms will try to compromise between themselves by maximizing the following
Nash product
\begin{equation}
  \label{e:bargaining}
  \max_{\lambda_i \in [0, 1]} (\varPi^e_1(\lambda_1,
  \lambda_2) - \varPi^e_1(0, 0)) (\varPi^e_2(\lambda_1, \lambda_2) -
  \varPi^e_2(0, 0)) \text{ s.t. } \forall i \ \varPi^e_i(\lambda_1, \lambda_2)
  - \varPi^e_i(0, 0) \ge 0.
\end{equation}
The Nash bargaining solution is a commonly used model for two-player
bargaining, as it results in the only possible compromise satisfying invariance
to the affine transformations of profits, Pareto optimality, independence of
irrelevant alternatives, and symmetry \citep{b86}. The following theorem
describes the firms' behavior in this setup (see proof in
\cref{sec:bargaining-proof}).

\begin{theorem}
  \label{t:bargaining}
  W.l.o.g. assume that $n_1 \ge n_2$. Additionally, assume $(1 - a) \gg b$ (in
  particular, $b < 5 (1 - a)$). The solution $(\lambda^*_1, \lambda^*_2)$ to the
  problem (\ref{e:bargaining}) is
  \begin{align*}
    \lambda_1^* &= \tilde{\lambda}_1 + \O\Par*{\frac{b}{1 - a}}, \ \lambda_2^* =
    1, \text{ where}\\
    \tilde{\lambda}_1 &= 
    \begin{cases}
      \frac{n_2}{n_1} \Par[\Big]{\Par[\Big]{1 - \frac{4 + \gamma^2}{4 \gamma}
      \Par*{\frac{n_2}{n_1}}^\beta \Par[\Big]{1 - \Par*{\frac{n_1}{n_1 +
      n_2}}^{\beta}}}^{-1/\beta} - 1},
      & \frac{4 + \gamma^2}{n_1^\beta} \le \frac{4 \gamma}{n_2^\beta} +
      \frac{(2 - \gamma)^2}{(n_1 + n_2)^\beta},\\
      1& \text{otherwise}.
    \end{cases}
  \end{align*}
  Moreover, $\tilde{\lambda}_1$ is a decreasing function in $\gamma$ and
  $\beta$, and $\tilde{\lambda}_1 = \frac{4 + \gamma^2}{4 \gamma}
  \Par*{\frac{n_2}{n_1}}^{\beta + 2} (1 + \o(1))$ when $n_2 / n_1 \to 0$.
\end{theorem}

\myparagraph{Discussion}
We can see that the results of \cref{t:bargaining} are similar to those in
\cref{sec:data_sharing}. When there is a big difference between the
amount of data ($n_1 \gg n_2$), the big firm does not collaborate with the
small firm ($\lambda^*_1 = \O((n_2/n_1)^{\beta + 2})$). However, when the
firms have similar amounts of data, the large firm will share almost all of it
($\lambda^*_1 \approx 1$), and the sharing proportion increases with a decrease
in competition (smaller $\gamma$) and an increase in learning task hardness
(smaller $\beta$). In contrast to the previous results, the smaller firm always
shares all its data with the competitor ($\lambda^*_2 = 1$).


\section{Data sharing between many firms}
\label{sec:many_firms}
In this section, we investigate the data-sharing problem for many companies.
For brevity, we only consider the case of Cournot competition. The key
challenge in the case of many companies is the variety of possible
coalitional structures. Moreover, since the companies may have conflicting
incentives and data sharing requires mutual agreement, it is not immediately
clear how to assess the ``plausibility'' of a particular structure.

We discuss one possible way to solve this problem using non-cooperative game
theory \citep{k18}. We assume a coalition formation process between the firms,
model it as a non-cooperative game, and study its Nash equilibria. The
non-cooperative approach often offers a unique \emph{sub-game perfect
equilibrium}, a standard solution notion in sequential games \citep{m95}.
However, this equilibrium depends on the assumptions about the bargaining
process. For this reason, in \cref{sec:other-coals} we also study alternative
formulations: a cooperative setting and two non-cooperative settings with only
one non-singleton coalition.

\myparagraph{Non-cooperative data-sharing game} We order the firms in a
decreasing data size order so that firm $F_1$ has the largest dataset. In this
order, the firms propose forming a coalition with their competitors. First, the
largest company makes up to $2^{m - 1}$ proposals. For any offer, each invited
firm accepts or declines it in the decreasing data size order. If anyone
disagrees, the coalition is rejected, and the first firm makes another offer.
Otherwise, the coalition is formed, and all companies in it leave the game.
After the first company forms a coalition or exhausts all of its offers, the
second firm, in the decreasing data size order, begins to propose, etc.

We use the standard backward induction procedure \citep{m95} to calculate the
sub-game perfect equilibrium of this game. This method solves the game from the
end to the beginning. First, the algorithm looks at all possible states one
step before the end of the game. Since the actions of the last player are
rational, the algorithm can identify these actions accurately. This way, the
algorithm moves one step earlier in the coalition formation process and can now
identify the action taken before this state occurred. The procedure iterates
until the start of the game is reached.

\begin{figure}[b]
  \centering
  \scalebox{0.5}{\input{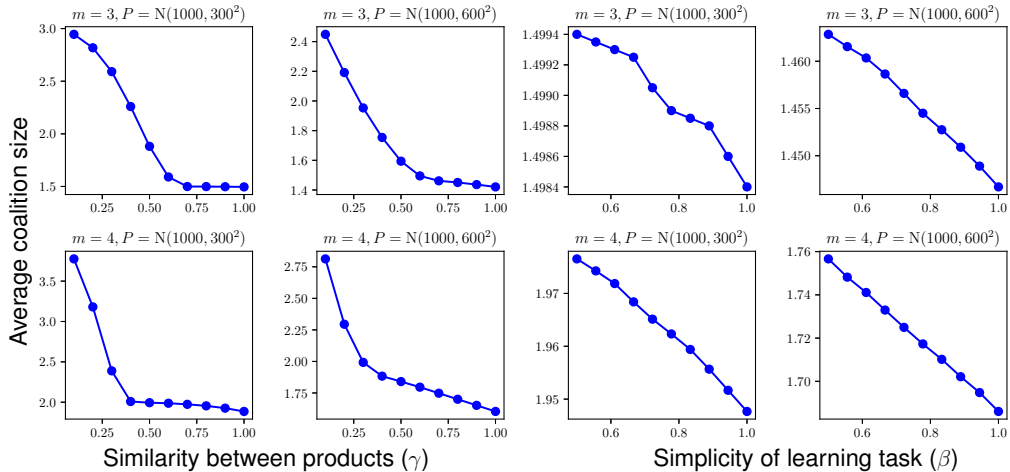}}
  \caption{The dependence of the average coalition size on $\gamma$ (left part)
  and $\beta$ (right part)
  in the synthetic experiments. The $y$-axes report the mean of the average
  size of the coalitions in the equilibrium partition, where mean is taken over
  $10000$ Monte Carlo simulations of the game. See the main text for details.}
  \label{f:experiment}
\end{figure}

\myparagraph{Experiments} We use the procedure described above to empirically
test the conclusions of the previous sections. We sample $m$ dataset sizes, one
for each firm, from a distribution $P = \N(\mu, \sigma^2)$ clipped at $1$ form
below. Then, we calculate the equilibrium of the resulting data-sharing game
and compute the average size (number of companies) of the resulting coalitions
at the equilibrium, as a measure of the extend to which companies collaborate.

We repeat the experiment $10000$ times, for fixed values of $m, \gamma, \beta,
\mu, \sigma$ and compute the mean of the average coalition size over these
runs. Our simulation solves each instance of the data-sharing game exactly and
average it over a big number of independent runs, which makes our results very
precise.

Repeating the experiment for different values of $\gamma$ and $\beta$, we
can observe the dependence of this average coalition size on the two parameters
of interest. In \cref{f:experiment} we plot these dependencies, for different
values of $m \in \{3,4\}$ and $\sigma \in \{300, 600\}$. When varying one of
these parameters, the default values for the other one is $\gamma = 0.8$ and
$\beta = 0.9$. We also provide results for other values of $m$ and $\sigma$ in
\cref{sec:add-coop}.

As we can see, the conclusion of \cref{t:main} transfer to this experiment: the
average size of the coalitions and thus cooperation incentives, are decreasing
in both $\gamma$ and $\beta$, across all experimental setups.

\section{Conclusion}
\label{sec:conclusion}
In this work, we introduced and studied a framework for data-sharing among
market competitors. The framework is modular and consists of three
components---market model, data impact, and collaboration scheme---which can be
specified for any particular application. To examine the effects of various
parameters, we instantiated the framework using a conventional market model and
a natural model of data impact grounded in learning theory. We then studied
several data-sharing games within these models.

Our findings indicate that the characteristics of competition may have a
significant effect on the data-sharing decisions. Specifically, we found that
higher product differentiation generally increases the willingness for
collaboration, as does the complexity of the learning task. Interestingly, our
results suggests that the firms might not want to participate in the federated
learning even in the absence of privacy concerns. On the other hand, we also
predict that data-sharing collaborations might emerge between competitors even
without any external regulations. We hope our study will inspire further
in-depth investigations into the nuanced trade-offs in data sharing, allowing
competition and collaboration to coexist in data-driven environments.

\paragraph{Limitations and future work} While our general framework can
describe many data-sharing settings, the quantitative results rely on several
important assumptions. First, we assume that the firms' data are homogeneous.
We expect that designing a suitable model in the case of significant
heterogeneity  is a significant challenge orthogonal to the focus of our work
\citep[e.g.,][]{gulrajani2020search}. Moreover, in some cases, additional data
from a different distribution may damage model performance, yielding orthogonal
data-sharing considerations \citep{d21}.

Second, we assume that legal constraints are not present or are automatically
fulfilled in all our settings. However, we see the evaluation of AI regulatory
frameworks as an important direction for future work. The same applies to
training cost considerations.

Finally, we do not consider sociological aspects, such as reputational effects
(e.g., reputation loss due to a refusal to share data or a poorly communicated
decision to share data). Due to the inherent non-rivalry of data \citep{k22},
sociological modeling, e.g., the theory of collective action, may provide
further valuable insights into the data-sharing problem.

We see our contributions as mostly conceptual and do not aim to provide a fully
realistic model that can directly inform practitioners about the benefits of
data-sharing decisions. However, we hope our results will remain qualitatively
valid in real-world settings since we used a market model widely adopted in the
economic theoretical literature \citep{c19} and a data impact model motivated
by established theoretical frameworks in machine learning \citep{t04}. Also, we
hope that the three-component decomposition of the data-sharing problem will
help practitioners leverage their situation-specific knowledge of their market
and ML models to build more accurate models for their specific needs.

\section*{Acknowledgments}

This research was partially funded by the Ministry of Education and Science of Bulgaria (support for
INSAIT, part of the Bulgarian National Roadmap for Research Infrastructure). The authors would
like to thank Florian Dorner, Mark Vero, Nikola Jovanovic and Sam Motamed for providing helpful
feedback on earlier versions of this work.

\bibliography{bibliography}

\begin{thebibliography}{45}
\providecommand{\natexlab}[1]{#1}
\providecommand{\url}[1]{\texttt{#1}}
\expandafter\ifx\csname urlstyle\endcsname\relax
  \providecommand{\doi}[1]{doi: #1}\else
  \providecommand{\doi}{doi: \begingroup \urlstyle{rm}\Url}\fi

\bibitem[Amir et~al.(2017)Amir, Erickson, and Jin]{a17m}
Amir, R., Erickson, P., and Jin, J.
\newblock On the microeconomic foundations of linear demand for differentiated
  products.
\newblock \emph{Journal of Economic Theory}, 169:\penalty0 641--665, 2017.

\bibitem[Bergemann \& Bonatti(2019)Bergemann and Bonatti]{be19}
Bergemann, D. and Bonatti, A.
\newblock Markets for information: An introduction.
\newblock \emph{Annual Review of Economics}, 11:\penalty0 85--107, 2019.

\bibitem[Berry et~al.(2019)Berry, Gaynor, and Scott~Morton]{b19}
Berry, S., Gaynor, M., and Scott~Morton, F.
\newblock Do increasing markups matter? {L}essons from empirical industrial
  organization.
\newblock \emph{Journal of Economic Perspectives}, 33\penalty0 (3):\penalty0
  44--68, 2019.

\bibitem[Bertrand(1883)]{b83}
Bertrand, J.
\newblock Book review of theorie mathematique de la richesse social and of
  recherches sur les principes mathematiques de la theorie des richesses.
\newblock \emph{Journal des savants}, 1883.

\bibitem[Binmore et~al.(1986)Binmore, Rubinstein, and Wolinsky]{b86}
Binmore, K., Rubinstein, A., and Wolinsky, A.
\newblock The nash bargaining solution in economic modelling.
\newblock \emph{RAND Journal of Economics}, 17\penalty0 (2):\penalty0 176--188,
  1986.

\bibitem[Blanchard et~al.(2017)Blanchard, El~Mhamdi, Guerraoui, and
  Stainer]{blanchard2017machine}
Blanchard, P., El~Mhamdi, E.~M., Guerraoui, R., and Stainer, J.
\newblock Machine learning with adversaries: Byzantine tolerant gradient
  descent.
\newblock In \emph{Advances in Neural Information Processing Systems},
  volume~30, 2017.

\bibitem[Blum et~al.(2021)Blum, Haghtalab, Phillips, and Shao]{b21}
Blum, A., Haghtalab, N., Phillips, R.~L., and Shao, H.
\newblock One for one, or all for all: Equilibria and optimality of
  collaboration in federated learning.
\newblock In \emph{International Conference on Machine Learning}, pp.\
  1005--1014. PMLR, 2021.

\bibitem[Bubeck(2015)]{bubeck2015convex}
Bubeck, S.
\newblock Convex optimization: Algorithms and complexity.
\newblock \emph{Foundations and Trends{\textregistered} in Machine Learning},
  8\penalty0 (3-4):\penalty0 231--357, 2015.

\bibitem[Carlton \& Perloff(1990)Carlton and Perloff]{c90}
Carlton, D.~W. and Perloff, J.~M.
\newblock \emph{Modern industrial organization}.
\newblock Higher Education, 1990.

\bibitem[Chon{\'e} \& Linnemer(2019)Chon{\'e} and Linnemer]{c19}
Chon{\'e}, P. and Linnemer, L.
\newblock The quasilinear quadratic utility model: An overview.
\newblock \emph{CESifo Working Paper}, 2019.

\bibitem[Chui et~al.(2022)Chui, Hall, Mayhew, and Singla]{survey22}
Chui, M., Hall, B., Mayhew, H., and Singla, A.
\newblock The state of {AI} in 2022---and a half decade in review.
\newblock QuantumBlack By McKinsey, 2022.

\bibitem[Cournot(1838)]{c38}
Cournot, A.-A.
\newblock \emph{Recherches sur les principes math{\'e}matiques de la
  th{\'e}orie des richesses}.
\newblock chez L. Hachette, 1838.

\bibitem[Dixit(1979)]{d79}
Dixit, A.
\newblock A model of duopoly suggesting a theory of entry barriers.
\newblock \emph{Bell Journal of Economics}, 10\penalty0 (1):\penalty0 20--32,
  1979.

\bibitem[Donahue \& Kleinberg(2021)Donahue and Kleinberg]{d21}
Donahue, K. and Kleinberg, J.
\newblock Model-sharing games: Analyzing federated learning under voluntary
  participation.
\newblock In \emph{AAAI Conference on Artificial Intelligence}, volume 35(6),
  pp.\  5303--5311, 2021.

\bibitem[Durrant et~al.(2022)Durrant, Markovic, Matthews, May, Enright, and
  Leontidis]{d22}
Durrant, A., Markovic, M., Matthews, D., May, D., Enright, J., and Leontidis,
  G.
\newblock The role of cross-silo federated learning in facilitating data
  sharing in the agri-food sector.
\newblock \emph{Computers and Electronics in Agriculture}, 193:\penalty0
  106648, 2022.

\bibitem[FedAI()]{w19}
FedAI.
\newblock We{B}ank and {S}wiss {R}e signed {C}ooperation {M}o{U}.
\newblock URL
  \url{https://www.fedai.org/news/webank-and-swiss-re-signed-cooperation-mou/}.

\bibitem[Fried(1984)]{f84}
Fried, D.
\newblock Incentives for information production and disclosure in a duopolistic
  environment.
\newblock \emph{The Quarterly Journal of Economics}, 99\penalty0 (2):\penalty0
  367--381, 1984.

\bibitem[Gulrajani \& Lopez-Paz(2021)Gulrajani and
  Lopez-Paz]{gulrajani2020search}
Gulrajani, I. and Lopez-Paz, D.
\newblock In search of lost domain generalization.
\newblock In \emph{International Conference on Learning Representations}, 2021.

\bibitem[Jones(1983)]{j83}
Jones, D.
\newblock Statistical analysis of empirical models fitted by optimization.
\newblock \emph{Biometrika}, 70\penalty0 (1):\penalty0 67--88, 1983.

\bibitem[Kairouz et~al.(2021)Kairouz, McMahan, Avent, Bellet, Bennis, Bhagoji,
  Bonawitz, Charles, Cormode, and Cummings]{k21}
Kairouz, P., McMahan, H.~B., Avent, B., Bellet, A., Bennis, M., Bhagoji, A.~N.,
  Bonawitz, K., Charles, Z., Cormode, G., and Cummings, R.
\newblock Advances and open problems in federated learning.
\newblock \emph{Foundations and Trends{\textregistered} in Machine Learning},
  14\penalty0 (1--2):\penalty0 1--210, 2021.

\bibitem[Karimireddy et~al.(2022)Karimireddy, Guo, and Jordan]{k22}
Karimireddy, S.~P., Guo, W., and Jordan, M.~I.
\newblock Mechanisms that {I}ncentivize {D}ata {S}haring in {F}ederated
  {L}earning.
\newblock \emph{Workshop on Federated Learning at NeurIPS (FL-NeurIPS'22)},
  2022.

\bibitem[K{\'o}czy(2018)]{k18}
K{\'o}czy, L.~{\'A}.
\newblock Partition function form games.
\newblock \emph{Theory and Decision Library C}, 48:\penalty0 312, 2018.

\bibitem[Liu et~al.(2022)Liu, Zhang, Wang, and Yang]{liu22}
Liu, J., Zhang, G., Wang, K., and Yang, K.
\newblock Task-load-aware game-theoretic framework for wireless federated
  learning.
\newblock \emph{IEEE Communications Letters}, 27\penalty0 (1):\penalty0
  268--272, 2022.

\bibitem[Lyu et~al.(2020{\natexlab{a}})Lyu, Xu, Wang, and Yu]{lyu20b}
Lyu, L., Xu, X., Wang, Q., and Yu, H.
\newblock Collaborative fairness in federated learning.
\newblock In \emph{Federated Learning}, pp.\  189--204. Springer,
  2020{\natexlab{a}}.

\bibitem[Lyu et~al.(2020{\natexlab{b}})Lyu, Yu, Nandakumar, Li, Ma, Jin, Yu,
  and Ng]{lyu20}
Lyu, L., Yu, J., Nandakumar, K., Li, Y., Ma, X., Jin, J., Yu, H., and Ng, K.~S.
\newblock Towards fair and privacy-preserving federated deep models.
\newblock \emph{IEEE Transactions on Parallel and Distributed Systems},
  31\penalty0 (11):\penalty0 2524--2541, 2020{\natexlab{b}}.

\bibitem[Mas-Colell et~al.(1995)Mas-Colell, Whinston, and Green]{m95}
Mas-Colell, A., Whinston, M.~D., and Green, J.~R.
\newblock \emph{Microeconomic theory}.
\newblock Oxford University Press, 1995.

\bibitem[Nash(1951)]{n51}
Nash, J.
\newblock Non-cooperative games.
\newblock \emph{Annals of Mathematics}, 54\penalty0 (2), 1951.

\bibitem[Raith(1996)]{r96}
Raith, M.
\newblock A general model of information sharing in oligopoly.
\newblock \emph{Journal of economic theory}, 71\penalty0 (1):\penalty0
  260--288, 1996.

\bibitem[Reiss \& Wolak(2007)Reiss and Wolak]{r07}
Reiss, P.~C. and Wolak, F.~A.
\newblock Structural econometric modeling: Rationales and examples from
  industrial organization.
\newblock \emph{Handbook of econometrics}, 6:\penalty0 4277--4415, 2007.

\bibitem[Richardson et~al.(2020)Richardson, Filos-Ratsikas, and
  Faltings]{richardson20}
Richardson, A., Filos-Ratsikas, A., and Faltings, B.
\newblock Budget-bounded incentives for federated learning.
\newblock \emph{Federated Learning: Privacy and Incentive}, pp.\  176--188,
  2020.

\bibitem[Rieke et~al.(2020)Rieke, Hancox, Li, Milletari, Roth, Albarqouni,
  Bakas, Galtier, Landman, and Maier-Hein]{r20}
Rieke, N., Hancox, J., Li, W., Milletari, F., Roth, H.~R., Albarqouni, S.,
  Bakas, S., Galtier, M.~N., Landman, B.~A., and Maier-Hein, K.
\newblock The future of digital health with federated learning.
\newblock \emph{NPJ digital medicine}, 3\penalty0 (1):\penalty0 119, 2020.

\bibitem[Rust \& Zahorik(1993)Rust and Zahorik]{r93}
Rust, R.~T. and Zahorik, A.~J.
\newblock Customer satisfaction, customer retention, and market share.
\newblock \emph{Journal of retailing}, 69\penalty0 (2):\penalty0 193--215,
  1993.

\bibitem[Shalev-Shwartz \& Ben-David(2014)Shalev-Shwartz and
  Ben-David]{shalev2014understanding}
Shalev-Shwartz, S. and Ben-David, S.
\newblock \emph{Understanding {M}achine {L}earning: From {T}heory to
  {A}lgorithms}.
\newblock Cambridge University Press, 2014.

\bibitem[Shapiro(1989)]{s89}
Shapiro, C.
\newblock Theories of oligopoly behavior.
\newblock \emph{Handbook of industrial organization}, 1:\penalty0 329--414,
  1989.

\bibitem[Tirole(1988)]{t88}
Tirole, J.
\newblock \emph{The theory of industrial organization}.
\newblock MIT Press, 1988.

\bibitem[Tsybakov(2004)]{t04}
Tsybakov, A.~B.
\newblock Optimal aggregation of classifiers in statistical learning.
\newblock \emph{The Annals of Statistics}, 32\penalty0 (1):\penalty0 135--166,
  2004.

\bibitem[Tu et~al.(2022)Tu, Zhu, Luong, Niyato, Zhang, and Li]{tu22}
Tu, X., Zhu, K., Luong, N.~C., Niyato, D., Zhang, Y., and Li, J.
\newblock Incentive mechanisms for federated learning: From economic and game
  theoretic perspective.
\newblock \emph{IEEE transactions on cognitive communications and networking},
  8\penalty0 (3):\penalty0 1566--1593, 2022.

\bibitem[Vives(1984)]{v84}
Vives, X.
\newblock Duopoly information equilibrium: Cournot and {B}ertrand.
\newblock \emph{Journal of economic theory}, 34\penalty0 (1):\penalty0 71--94,
  1984.

\bibitem[Von~Neumann \& Morgenstern(1944)Von~Neumann and Morgenstern]{v44}
Von~Neumann, J. and Morgenstern, O.
\newblock Theory of games and economic behavior.
\newblock In \emph{Theory of games and economic behavior}. Princeton University
  Press, 1944.

\bibitem[Werner et~al.(2022)Werner, He, Karimireddy, Jordan, and Jaggi]{w22}
Werner, M., He, L., Karimireddy, S.~P., Jordan, M., and Jaggi, M.
\newblock Towards {P}rovably {P}ersonalized {F}ederated {L}earning via
  {T}hreshold-{C}lustering of {S}imilar {C}lients.
\newblock In \emph{Workshop on Federated Learning at NeurIPS (FL-NeurIPS'22)},
  2022.

\bibitem[Wu \& Yu(2022)Wu and Yu]{wu22}
Wu, X. and Yu, H.
\newblock Mar{S}-{FL}: Enabling {C}ompetitors to {C}ollaborate in {F}ederated
  {L}earning.
\newblock \emph{IEEE Transactions on Big Data}, \penalty0 (1):\penalty0 1--11,
  2022.

\bibitem[Yang et~al.(2020)Yang, Fan, and Yu]{y20}
Yang, Q., Fan, L., and Yu, H.
\newblock \emph{Federated Learning: Privacy and {I}ncentive}, volume 12500.
\newblock Springer Nature, 2020.

\bibitem[Yu et~al.(2020)Yu, Liu, Liu, Chen, Cong, Weng, Niyato, and Yang]{yu20}
Yu, H., Liu, Z., Liu, Y., Chen, T., Cong, M., Weng, X., Niyato, D., and Yang,
  Q.
\newblock A sustainable incentive scheme for federated learning.
\newblock \emph{IEEE Intelligent Systems}, 35\penalty0 (4):\penalty0 58--69,
  2020.

\bibitem[Zeng et~al.(2021)Zeng, Zeng, Wang, Li, and Chu]{z21}
Zeng, R., Zeng, C., Wang, X., Li, B., and Chu, X.
\newblock A comprehensive survey of incentive mechanism for federated learning.
\newblock \emph{arXiv preprint arXiv:2106.15406}, 2021.

\bibitem[Zhan et~al.(2021)Zhan, Zhang, Hong, Wu, Li, and Guo]{zhan21}
Zhan, Y., Zhang, J., Hong, Z., Wu, L., Li, P., and Guo, S.
\newblock A survey of incentive mechanism design for federated learning.
\newblock \emph{IEEE Transactions on Emerging Topics in Computing}, 10\penalty0
  (2):\penalty0 1035--1044, 2021.

\end{thebibliography}
\bibliographystyle{icml2023}

\clearpage

\appendix

\begin{center}
  {\LARGE Supplementary Material}
\end{center}

The supplementary material is structured as follows.
\begin{itemize}
  \item \cref{sec:proofs} contains the proofs of all results in the main text.
  \item \cref{sec:example} contains further examples of the data-sharing
    problem and costs.
  \item \cref{s:robustness} contains several extensions of the setting in
    \cref{sec:data_sharing}.
  \item \cref{sec:other-coals} contains alternative models of coalition
    formation between many firms.
  \item \cref{sec:welfare} contains the welfare analysis of the data-sharing
    problem between two firms.
  \item \cref{sec:add-coop} presents additional simulations for
    \cref{sec:many_firms}.
\end{itemize}

\section{Proofs}
\label{sec:proofs}

\subsection{Proof of \texorpdfstring{\cref{l:demand}}{Lemma \ref{l:demand}}}
\label{sec:demand-proof}

First, we will solve the same consumption problem without positivity
constraints
\[
  \max_{g, q_1, \dots, q_m} u(g, q_1, \dots, q_m) \text{ s.t. } g + \sum_{i =
  1}^m p_i q_i \le B.
\]
Using the KKT conditions for $q_i$ we get
\[
  0 = \pdv{L}{q_i} = \pdv{u}{q_i} - \lambda p_i,
\]
where $\lambda$ is the Lagrange multiplier for the budget constraint. The
first-order condition for $g$ gives
\[
  0 = 1 - \lambda \implies \lambda = 1.
\]
Therefore, the first-order conditions for $q_1, \dots, q_m$ look like
\[
  0 = 1 - q_i - \sum_{j \neq i} q_j - p_i \implies p_i = 1 - q_i - \gamma
  \sum_{j \neq i} q_j.
\]
Notice that
\[
  \sum_i p_i = m - \sum_i q_i - \gamma (m - 1) \sum_i q_i \implies \sum_i
  q_i = \frac{m - \sum_i p_i}{1 + \gamma (m - 1)}.
\]
Therefore,
\[
  q_i = \frac{(1 + \gamma (m - 1)) (1 - p_i) - \gamma (m - \sum_i p_i)}{(1 -
  \gamma) (1 + \gamma (m - 1))} =\\
  \frac{1 - \gamma - (1 + \gamma (m - 2)) p_i -
  \gamma \sum_{j \neq i} p_j}{(1 - \gamma) (1 + \gamma (m - 1))}.
\]
This solution is a global maximizer because $u(g, q_1, \dots, q_m)$ is
strongly concave in $q_1, \dots, q_m$ and linear in $g$.

Since the auxiliary problem has less constraints than the original problem, a
solution to the auxiliary problem will be a solution to the original problem as
long as all quantities are non-negative. These constraints give the following
restrictions on $p_1, \dots, p_m$ and $B$
\[
  \forall i \ (1 - \gamma) (1 - p_i) \ge \gamma \sum_{j \neq i} (p_j - p_i),
\]
\[
  B \ge \sum_{i = 1}^m q_i^* p_i.
\]
The first series of inequalities hold when prices are less than one and close
enough. In our setting, it is achieved when $1 > a$ and $(1 - a) \gg b$, where
$a$ and $b$ are parameters from \cref{e:data-dep} (since Equations
(\ref{eq:cournot-restriction}) and (\ref{eq:bertrand-restriction}) hold). The
last equation holds when $B$ is big enough, for example, when $B > m$.

\subsection{Proof of \texorpdfstring{\cref{l:cournot}}{Lemma \ref{l:cournot}}}
\label{sec:cournot-proof}

The profit maximization tasks of the firms are
\[
  \max_{q_i} \E((p_i - c_i) q_i) = \E(1 - \gamma \sum_{j \neq i} q_j - c_i) q_i
  - q_i^2.
\]
Using the first-order conditions, we get
\[
  q_i^*(q_{-i}) = \frac{\E(1 - c_i - \gamma \sum_{j \neq i} q_j)}{2}.
\]
Both equilibrium quantities should be the best responses to the opponent's
expected amounts
\[
  q_i^* = \frac{\E(1 - c_i - \gamma \sum_{j \neq i} q_j^*)}{2}.
\]
Therefore,
\[
  \sum_i q_i^* = \frac{\E(m - \sum_i c_i - (m - 1) \gamma \sum_i q_i^*)}{2},
\]
Notice that
\[
  \E (q_i^*) = \frac{\E \E (\dots)}{2} = \frac{\E (\dots)}{2} = q_i^*.
\]
So,
\[
  \sum_i q_i^* = \frac{m - \sum_i c_i^e}{2 + \gamma (m - 1)}.
\]
The best response equation gives
\[
  q_i^* = \frac{1 - c_i^e - \gamma \sum_j q_j^* + \gamma q_i^*}{2} \implies
  q_i^* = \frac{2 - \gamma - (2 + \gamma (m - 2)) c_i^e + \gamma \sum_{j \neq
  i} c_j^e}{(2 - \gamma) (2 + (m - 1) \gamma)}.
\]

We get some expression for the quantities produced by the firms. However, we
need to ensure that these quantities are positive. This property is equivalent
to the following inequalities
\begin{equation}
  \label{eq:cournot-restriction}
  \forall i \ (2 - \gamma) (1 - c_i^e) \ge \gamma \sum_{j \neq i} (c_i^e -
  c_j^e).
\end{equation}
These inequalities hold, for example, when $1 > a$ and $(1 - a) \gg b$, where
$a$ and $b$ are parameters from \cref{e:data-dep}.

Now, notice that the expected profit function is a second-degree polynomial in
$q_i$
\[
  \varPi^e_i = -q_i^2 + \beta q_i.
\]
Since the maximum value is achieved at $q_i^*(q_{-i})$, this polynomial looks
like
\[
  \varPi^e_i = -(q_i - q_i^*(q_{-i}))^2 + q_i^*(q_{-i})^2.
\]
And it gives the following equilibrium expected profits
\[
  \varPi_i^e = (q_i^*)^2.
\]

\subsection{Proof of \texorpdfstring{\cref{l:bertrand}}{Lemma
\ref{l:bertrand}}}
\label{sec:bertrand-proof}

The expected profit maximization tasks of the firms are
\[
  \max_{p_i} \E \Par*{(p_i - c_i) \Par*{\frac{1 - \gamma - (1 + \gamma (m - 2))
  p_i + \gamma \sum_{j \neq i} p_j}{(1 - \gamma) (1 + \gamma (m - 1))}}}.
\]
Thus, the best responses are
\[
  p_i^*(p_{-i}) = \frac{\E(1 - \gamma + (1 + \gamma (m - 2)) c_i + \gamma
  \sum_{j \neq i} p_j)}{2 (1 + \gamma (m - 2))}.
\]
Notice that
\[
  \E (p_i^*) = \frac{\E \E (\dots)}{2 (1 + \gamma (m - 2))} = \frac{\E
  (\dots)}{2 (1 + \gamma (m - 2))} = p_i^*.
\]
Thus,
\[
  p_i^* = \frac{1 - \gamma + (1 + \gamma (m - 2)) c^e_i + \gamma \sum_{j \neq
  i} p^*_j}{2 (1 + \gamma (m - 2))}.
\]
If we sum these equations, we get
\begin{align*}
  \sum_i p_i^* &= \frac{(1 - \gamma) m + (1 + \gamma (m - 2)) \sum_i c^e_i +
  \gamma (m - 1) \sum_i p^*_i}{2 (1 + \gamma (m - 2))} \implies\\
  \sum_i p_i^* &= \frac{(1 - \gamma) m + (1 + \gamma (m - 2)) \sum_i c_i^e}{2 +
  \gamma (m - 3)}.
\end{align*}
After substitution of this sum into the best response equation, we have
\begin{align*}
  p_i^* &= \frac{d_1 + d_2 c_i^e +  d_3 \sum_{j \neq i} c_j^e}{d_4},\\
  d_1 &= 2 + \gamma (2 m - 5) - \gamma^2 (2 m - 3),\\
  d_2 &= 2 + 3 \gamma (m - 2) + \gamma^2 (m^2 - 4 m + 4),\\
  d_3 &= \gamma + \gamma^2 (m - 2),\\
  d_4 &= 4 + 6 \gamma (m - 2) + \gamma^2 (2 m^2 - 9 m + 9). 
\end{align*}

Here, we again need to ensure that quantities produced by the firms are
positive. This property is equivalent the following inequalities
\begin{equation}
  \label{eq:bertrand-restriction}
  \forall p_i - c_i^e \ge 0 \iff d_1 (1 - c_i^e) \ge d_3 \sum_{j \neq i} (c_i^e
  - c_j^e).
\end{equation}
These inequalities hold, for example, when $1 > a$ and $(1 - a) \gg b$, where
$a$ and $b$ are parameters from \cref{e:data-dep}.

Notice that the expected profit of the firm is a second degree polynomial
\[
  \varPi^e_i = -\alpha p_i^2 + 2 \beta p_i - \gamma.
\]
Since best response $p_i^*(p_{-i})$ is a maximizer of this polynomial, we get
\[
  \varPi^e_i = -\alpha (p_i - p_i^*(p_{-i}))^2 + \delta.
\]
Also, notice that
\[
  \varPi^e_i(c_i^e, p_{-i}) = \E((c_i^e - c_i) q_i(c_i^e, p_{-i})) = 0.
\]
Therefore,
\[
  \varPi^e_i = \alpha ((c_i^e - p_i^*(p_{-i}))^2 - (p_i - p_i^*(p_{-i}))^2).
\]
Direct calculations show that
\[
  \alpha = \frac{1 + \gamma (m + 2)}{(1 - \gamma) (1 + \gamma (m - 1))}.
\]
These equations imply the following equilibrium profits
\[
  \varPi^e_i(p_i^*, p_{-i}^*) = \frac{(1 + \gamma (m + 2)) (p_i^* - c_i^e)^2}{(1 -
  \gamma) (1 + \gamma (m - 1))},
\]

\subsection{Derivation of \texorpdfstring{\cref{e:mle}}{Equation
(\ref{e:mle})}}
\label{sec:asymptotic_normality}

Formally, we assume that the cost of production of one unit of good $C$ is a
random variable that depends on random noise $\X \in \mathbb{R}^m$ and the
decisions of a firm $\s \in \mathbb{R}^n$. The firm knows that the noise is
distributed according to a density function
\[
  \X \sim f(\cdot, \prm), \prm \in \mathbb{R}^d.
\]
However, it does not know the parameters of the distribution $\prm$.

The firm wants to minimize its marginal cost on average
\[
  c(\s) = \E_{\X \sim f(\cdot, \prm)}(C(\s, \X)).
\]
To achieve it, the firm decides to use the maximum likelihood estimates
$\est{\prm}$ of $\prm$ using its data about the realizations of the noise.
After that, the firm chooses an action that minimizes an estimated cost on
average
\[
  \s_\text{fin} = \argmin_{\s} \E_{\X \sim f(\cdot, \est{\prm})}(C(\s, \X)).
\]
The resulting cost $c(\s_\text{fin})$ is a random variable because of
randomness of the firm's sample $S$. However, the asymptotic normality of MLE
allows us to reason about a cost distribution.

We use the delta method to achieve this goal. By the asymptotic normality,
\[
  \sqrt{n} (\est{\prm} - \prm) \dto \N(0, \mat{I}(\prm)^{-1}).
\]
Further, by the delta method,
\[
  \sqrt{n} (\s_\text{fin} - \s^*) \dto \sqrt{n} \pdv{\s}{\prm}^\tran
  (\est{\prm} - \prm).
\]
Finally, by the second-order delta method,
\begin{align*}
  n (c(\s_\text{fin}) - c(\s^*)) &\dto n \nabla c(\s^*)^\tran
  (\s_\text{fin} - \s^*) + n (\s_\text{fin} - \s^*)^\tran
  \frac{\nabla^2 c(\s^*)}{2}
  (\s_\text{fin} - \s^*) \\
  &= n (\s_\text{fin} - \s^*)^\tran \frac{\nabla^2 c(\s^*)}{2}
  (\s_\text{fin} - \s^*).
\end{align*}
Thus, the marginal cost will be distributed approximately as generalized
chi-squared distribution. The expected marginal cost will be approximately
\[
  c^e = \E_{S}(c(\s_\text{fin})) \approx c(\s^*) + \frac{1}{2 n}
  \Tr \Par*{\nabla^2 c(\s^*) \pdv{\s}{\prm}^\tran
  \mat{I}(\prm)^{-1} \pdv{\s}{\prm}}.
\]
Therefore, one of the natural choices for the dependency of expected marginal
cost on the number of data points is
\[
  c^e(n) = a + \frac{b}{n}.
\]

\subsection{Proof of \texorpdfstring{\cref{t:main}}{Theorem \ref{t:main}}}
\label{sec:main_proof}

In the case $\gamma \le 0$, the collaboration is always profitable because the
right-hand side of the criterion is less than zero, and the left-hand side of
the criterion is not less than zero.

In the case $\gamma > 0$, the criteria for the firm $i$ are the following. In
the case of Cournot competition, we have
\begin{multline*}
  (2 - \gamma) (n_i^{-\beta} - (n_1 + n_2)^{-\beta}) > \gamma
  (n_{-i}^{-\beta} - n_i^{-\beta}) \\
  \iff f_\text{Cournot}(x, \gamma, \beta) \defeq 2 x^\beta - \gamma (1 -
  x)^\beta - (2 - \gamma) x^\beta (1 - x)^\beta > 0,
\end{multline*}
where $x = n_{-i} / (n_1 + n_2)$. The right-hand side of the criterion $f$ is
increasing in $x$, $f_\text{Cournot}(0, \gamma, \beta) < 0$, and
$f_\text{Cournot}(1, \gamma, \beta) > 0$. Thus, there exists a break-even
point $x_\text{Cournot}(\gamma, \beta)$ such that
\[
  f_\text{Cournot}(x, \gamma, \beta)
  \begin{cases}
    < 0, & x < x_\text{Cournot}(\gamma, \beta),\\
    = 0, & x = x_\text{Cournot}(\gamma, \beta),\\
    > 0, & x > x_\text{Cournot}(\gamma, \beta).\\
  \end{cases}
\]
Similarly, a break-even point exists in the Bertrand case.

Now, we prove the properties of $x$.

The first property follows from the fact that $f$ is decreasing in $\gamma$.
The second property follows from the fact that $f_\text{Cournot} >
f_\text{Bertrand}$. The proof of the third property is more complicated.

In the case of Cournot, the proof is the following. Let $x(\beta) =
x_\text{Cournot}(\gamma, \beta)$. We want to show that $x(\beta)$ is
increasing in $\beta$. To achieve it, we will show that
$f_\text{Cournot}(x(\beta), \gamma, \beta + \epsilon) < 0$. From the
definition,
\[
  x(\beta) = \Par*{\frac{\gamma (1 - x(\beta))^\beta}{2 - (2 - \gamma) (1 -
  x(\beta))^\beta}}^{\frac{1}{\beta}}.
\]
So,
\begin{multline*}
  f_\text{Cournot}(x(\beta), \gamma, \beta + \epsilon) < 0 \\
  \iff \Par*{\frac{\gamma (1 - x(\beta))^\beta}{2 - (2 - \gamma) (1 -
  x(\beta))^\beta}}^{\frac{1}{\beta}} < \Par*{\frac{\gamma (1 -
  x(\beta))^{\beta + \epsilon}}{2 - (2 - \gamma) (1 - x(\beta))^{\beta +
  \epsilon}}}^{\frac{1}{\beta + \epsilon}} \\
  \iff \frac{\gamma}{2} \Par*{\frac{2 - (2 - \gamma) (1 - x(\beta))^{\beta +
  \epsilon}}{\gamma}}^{\frac{\beta}{\beta + \epsilon}} + \frac{2 - \gamma}{2}
  (1 - x)^\beta < 1.
\end{multline*}
The last inequality follows from the concavity of the function
$x^{\frac{\beta}{\beta + \epsilon}}$. The proof is similar in the Bertrand
case.

\subsection{Proof of \texorpdfstring{\cref{t:bargaining}}{Theorem
\ref{t:bargaining}}}
\label{sec:bargaining-proof}

Direct computations show
\begin{align*}
  &\varPi^e_i(\lambda_1, \lambda_2) - \varPi^e_i(0, 0) =\\
  &(2 c^e_i(n_i) - \gamma c^e_{-i}(n_{-i}) - 2 c^e_i(n_i + n_{-i}
  \lambda_{-i}) + \gamma c^e_{-i}(n_{-i} + n_1 \lambda_i)) \times\\
  &\frac{(4 - 2 \gamma - 2
  c^e_i(n_i) + \gamma c^e_{-i}(n_{-i}) - 2 c^e_i(n_i + n_{-i} \lambda_{-i}) +
  \gamma c^e_{-i}(n_{-i} + n_i \lambda_i))}{(4 - \gamma^2)^2}.
\end{align*}
Denote
\[
  u_i = \frac{1}{n_i^\beta} - \frac{1}{(n_i + \lambda_{-i} n_{-i})^\beta}.
\]
It gives
\[
  c_i^e(n_i) - c_i^e(n_i + \lambda_{-i} n_{-i}) = a + \frac{b}{n_i^\beta} - a
  - \frac{b}{(n_i + n_{-i} \lambda_{-i})^\beta} = b u_i.
\]
So, the profit gain will be equal to
\[
  \varPi^e_i(\lambda_1, \lambda_2) - \varPi^e_i(0, 0) =
  (2 u_i - \gamma u_{-i}) \Par*{1 - \frac{4 \xi}{n_i^\beta} + \frac{2 \gamma
  \xi}{n_{-i}^\beta} + \xi (2 u_i - \gamma u_{-i})} \frac{b (4 - 2 \gamma) (1
  - a)}{(4 - \gamma^2)^2},
\]
where $\xi = \frac{b}{(4 - 2 \gamma) (1 - a)}$. Denote
\begin{align*}
  h_i &= \frac{1}{n_i^\beta} - \frac{1}{(n_1 + n_2)^\beta},\\
  g_i &= \frac{4}{n_i^\beta} - \frac{2 \gamma}{n_{-i}^\beta}.
\end{align*}
Then \cref{e:bargaining} will simplify to
\[
  \max_{u_i \in [0, h_i]} \! (2 u_1 - \gamma u_2) (2 u_2 - \gamma u_1) (1 - \xi
  g_1 + \xi (2 u_1 - \gamma u_2)) (1 - \xi g_2 + \xi (2 u_2 - \gamma u_1))
  \, \text{s.t.} \forall i \, 2 u_i - \gamma u_{-i} \! \ge \! 0.
\]
Now, we will change variables
\[
  v_i \defeq 2 u_i - \gamma u_{-i} \implies u_i = \frac{2 v_i + \gamma
  v_{-i}}{4 - \gamma^2},
\]
which will give the following problem
\[
  \max_{v_i \ge 0} v_1 v_2 (1 - \xi g_1 + \xi v_1) (1 - \xi g_2 + \xi v_2)
  \text{ s.t. } \forall i \ 2 v_i + \gamma v_{-i} \le (4 - \gamma^2) h_i.
\]

Notice that the restrictions $v_i \ge 0$ are not binding since the point
$(v_1, v_2) = (\epsilon, \epsilon)$ delivers positive Nash product and
satisfies the restrictions. However, one of the restrictions $2 v_i + \gamma
v_{-i} \le (4 - \gamma^2) h_i$ should be binding because any internal point
$(v_1, v_2)$ could be transformed to the point $((1 + \epsilon) v_1, (1 +
\epsilon) v_2)$, which would increase the objective function and satisfy the
restrictions.

First, consider the case $2 v_1^* + \gamma v_2^* < (4 - \gamma^2) h_1$ at
an optimum $(v_1^*, v_2^*)$. According to the previous paragraph, we have $2
v_2^* + \gamma v_1^* = (4 - \gamma^2) h_2$, which gives $v_2^* - v_1^* > (2 +
\gamma) (h_2 - h_1) \ge 0$. The derivative of the objective along the
direction $(2, -\gamma)$ equals to
\[
  (2 v_2^* - \gamma v_1^*) (1 - \xi g_1 + \xi v_1^*) (1 - \xi g_2 + \xi
  v_2^*) + \xi v_1^* v_2^* (2 - \gamma - \xi (2 g_2 - \gamma g_1 - 2 v_2^* +
  \gamma v_1^*)).
\]
Since $2 v_2^* - \gamma v_1^* > 0$, $\abs{2 g_2 - \gamma g_1 - 2 v_2^* +
\gamma v_1^*} \le 10$, and $\xi \le \frac{b}{2 (1 - a)} \le \frac{1}{10}$,
the derivative is positive. Thus, the point $(v_1^*, v_2^*)$ can not be
optimal because a small shift in the direction $(2, -\gamma)$ would increase
the objective and would not violate the restrictions. Therefore, in the
optimum, we will have $2 v_1^* = (4 - \gamma^2) h_1 - \gamma v_2^*$.

If we substitute the last expression into the objective function, the
objective will become a fourth-degree polynomial of $v_2$. If the stationary
point of the objective function $v_2^s$ satisfy the restriction $2 v_2^s +
\gamma v_1(v_2^s) \le (4 - \gamma^2) h_2$, it will be a solution to the
original problem. Otherwise, the solution will come form the equation $2
v_2^0 + \gamma v_1(v_2^0) = (4 - \gamma^2) h_2$ since all other restrictions
are not binding. So, the solution will be equal to $v_2^* = \min(v_2^s,
v_2^0)$ (the restriction $2 v_2 + \gamma v_1(v_2) \le (4 - \gamma^2) h_2$
does not hold only when $v_2$ is big).

Consider the following relaxed problem
\[
  \max_{v_i \ge 0} v_1 v_2 (1 - \xi g_1 + \xi v_1) (1 - \xi g_2 + \xi v_2)
  \text{ s.t. } 2 v_1 + \gamma v_2 = (4 - \gamma^2) h_1.
\]
The stationary point will be the root of the derivative and hence will satisfy
the following equation
\begin{multline*}
  (\gamma v_2^s - 2 v_1(v_2^s)) (1 - \xi g_1 + \xi v_1(v_2^s)) (1 - \xi g_2
  + \xi v_2^s) =\\
  \xi v_1(v_2^s) v_2^s (2 (1 - \xi g_1 + \xi v_1(v_2^s)) - \gamma (1 - \xi g_2
  + \xi v_2^s))\\
  \implies \gamma v_2^s - 2 v_1(v_2^s) = \xi \frac{v_1(v_2^s) v_2^s (2 (1 -
  \xi g_1 + \xi v_1(v_2^s)) - \gamma (1 - \xi g_2 + \xi v_2^s))}{(1 - \xi g_1
  + \xi v_1(v_2^s)) (1 - \xi g_2 + \xi v_2^s)}.
\end{multline*}
Since $-2 \le g_i \le 4$, $0 \le v_i \le 2$, and $v_1(v_2) v_2 \le \frac{(4 -
\gamma^2)^2 h_1^2}{8 \gamma}$, we get
\[
  \abs{\gamma v_2^s - 2 v_1(v_2^s)} \le \xi \frac{(4 - \gamma^2) h_1^2}{8
  \gamma} \frac{2 - \gamma + 12 \xi}{(1 - 4 \xi)^2} = \O(\xi).
\]
So, the stationary point is approximately $v_2^s = \frac{(4 - \gamma^2)
h_1}{2 \gamma} + \O(\xi)$.

Therefore, the solution to the original problem is approximately
\[
  v_2^* = \min\Par*{\frac{(4 - \gamma^2) h_1}{2 \gamma} + \O(\xi), 2 h_2 -
  \gamma h_1} = \min\Par*{\frac{(4 - \gamma^2) h_1}{2 \gamma}, 2 h_2 - \gamma
  h_1} + \O(\xi).
\]
Substituting everything back, we get the solution presented in the theorem
statement.

Now, we will prove that $\tilde{\lambda}_1$ is non-increasing in $\gamma$.

First, the criterion function for choosing non-unit solution
\[
  \frac{(2 - \gamma)^2}{(n_1 + n_2)^\beta} + \frac{4 \gamma}{n_2^\beta} -
  \frac{4 + \gamma^2}{n_1^\beta} = 4 \gamma h_2 - (4 + \gamma^2) h_1.
\]
is increasing in $\gamma$ since its derivative $4 h_2 - 2 \gamma h_1 \ge 0$
is positive. So, when $\gamma$ increases non-unit solution becomes more
probable.

Second, the non-unit solution is decreasing in $\gamma$ since the function
$\frac{4 + \gamma^2}{\gamma} = \frac{4}{\gamma} + \gamma$ is decreasing in
$\gamma$ on $(0, 1)$. (Note that$\tilde{\lambda}_1 = 1$, when $\gamma \le
0$.)

Now, we will prove that $\tilde{\lambda}_1$ is non-increasing in $\beta$.

First, notice that the non-unit criterion is equivalent to
\[
  f(x, \beta) = 4 \gamma (1 - x)^\beta (1 - x^\beta) - (4 + \gamma^2) x^\beta
  (1 - (1 - x)^\beta) \ge 0,
\]
where $x = \frac{n_2}{n_1 + n_2}$. The criterion function $f$ is decreasing
in $x$ and have one root $x(\beta)$ on the interval $[0, 1]$. We will prove
that the root is shifting to the right, making more size pairs produce
non-unit solution. Notice the criterion can be also expressed as
\[
  x \le \Par*{\frac{4 \gamma (1 - x)^\beta}{4 + \gamma^2 - (2 - \gamma)^2 (1
  - x)^\beta}}^{\frac{1}{\beta}}.
\]
So, to prove that $x(\beta) \le x(\beta + \epsilon)$ it is sufficient to
show that
\begin{multline*}
  \Par*{\frac{4 \gamma (1 - x(\beta))^\beta}{4 + \gamma^2 - (2 - \gamma)^2 (1
  - x(\beta))^\beta}}^{\frac{1}{\beta}} = x(\beta) < \Par*{\frac{4 \gamma (1
  - x(\beta))^{\beta + \epsilon}}{4 + \gamma^2 - (2 - \gamma)^2 (1 -
  x(\beta))^{\beta + \epsilon}}}^{\frac{1}{\beta + \epsilon}}\\
  \iff \frac{(2 - \gamma)^2}{4 + \gamma^2} (1 - x)^\beta + \frac{4 \gamma}{4
  + \gamma^2} \Par*{\frac{4 + \gamma^2 - (2 - \gamma)^2 (1 - x)^{\beta +
  \epsilon}}{4 \gamma}}^{\frac{\beta}{\beta + \epsilon}} \le 1.
\end{multline*}
The last inequality follows from the concavity of the function
$x^{\frac{\beta}{\beta + \epsilon}}$.

Now, we will prove that the non-unit solution is decreasing in $\beta$ given
that non-unit criterion holds. Notice that the non-unit solution
$\tilde{\lambda}_1(\beta)$ satisfies
\[
  \frac{1}{n_2^\beta} - \frac{1}{(n_2 + \tilde{\lambda}_1(\beta) n_1)^\beta}
  = \frac{4 + \gamma^2}{4 \gamma} \Par*{\frac{1}{n_1^\beta} - \frac{1}{(n_1 +
  n_2)^\beta}}.
\]
So, to show that $\tilde{\lambda}_1(\beta) \ge \tilde{\lambda}_1(\beta +
\epsilon)$, it is sufficient to show that
\[
  (x^{-\beta} - \delta (1 - x)^{-\beta} + \delta)^{\frac{1}{\beta}} \le
  (x^{-\beta - \epsilon} - \delta (1 - x)^{-\beta - \epsilon} +
  \delta)^{\frac{1}{\beta + \epsilon}},
\]
where $x = \frac{n_2}{n_1 + n_2}$ and $\delta = \frac{4 + \gamma^2}{4
\gamma}$. This inequality is equivalent to
\[
  ((1 - x)^\beta x^{-\beta} - \delta + \delta (1 - x)^\beta)^{\frac{\beta +
  \epsilon}{\beta}} \le (1 - x)^{\beta + \epsilon} x^{-\beta - \epsilon} -
  \delta + \delta (1 - x)^{\beta + \epsilon}.
\]
Denote $u_0 = \Par*{\frac{1 - x}{x}}^\beta$, $v_0 = (1 - x)^\beta$, and
$\alpha = \frac{\beta + \epsilon}{\beta}$. We get the following inequality
\[
  (u_0 - \delta + \delta v_0)^\alpha \le u_0^\alpha - \delta + \delta
  v_0^\alpha.
\]
Denote
\[
  f(u, v) = u^\alpha - \delta + \delta v^\alpha - (u - \delta + \delta
  v)^\alpha.
\]
Notice that $f(u, 1) = 0$. Additionally notice that
\[
  \pdv{f}{v} = \delta \alpha v^{\alpha - 1} - \delta \alpha (u - \delta +
  \delta v)^{\alpha - 1} \le 0 \iff v \le u - \delta + \delta v.
\]
If we prove that this derivative is negative for all points $(u_0, x)$, where
$x \in (v_0, 1)$, we will prove the desired result. Notice that the last
inequality is the most stringent for the point $(u_0, v_0)$. At this point,
we have
\[
  v_0 \le u_0 - \delta + \delta v_0 \iff 0 \le \frac{4 \gamma}{n_2^\beta} -
  \frac{4 + \gamma^2}{n_1^\beta} + \frac{(2 - \gamma)^2}{(n_1 + n_2)^\beta}.
\]
The last inequality holds because we consider the non-unit solution.

The last property of the theorem follows from the fact that, for small ratio
$\frac{n_2}{n_1}$, the non-unit criterion holds and the Taylor formula for
this ratio results in the expression presented in the statement.

\section{Examples of data-sharing problem}
\label{sec:example}

\subsection{Taxi rides}

In this section, how our framework works using motivating example from
\cref{sec:setting}. In this example, two taxi firms use data to optimize the
driver scheduling algorithm. This optimization allow them to use less driver
time; thus, making their expected marginal costs lower.

We assume that the only uncertainty arising in scheduling algorithm is the
duration of one kilometer ride $X \sim \N(\mu, 1)$. The firms do not know the
parameter $\mu$. However, each firm $F_i$ has $n_i$ observations of this random
variable $S_i = \{X_i^j\}_{j=1}^{n_i}$.

Their scheduling algorithm require the estimate $s$ of ride duration $X$. If
the estimate is too low $s < X$, the company will need to attract more drivers
and pay them more. If the estimate is too big $s > X$, the drivers will be
underutilized. We assume that costs associated with these losses have MSE form
$C(s, X) = a - b + b (s - X)^2$. This functional form will imply the following
costs on average
\[
  c(s) = \E_X(C(s, X)) = a - b + b \E((s - X)^2) = a + b (s - \mu)^2.
\]

If the firms do not collaborate, $F_i$ uses an average of its own points
\[
  s_{i, \text{ind}} = \frac{1}{n_i} \sum_{j = 1}^{n_i} X^i_j \implies s_{i,
  \text{ind}} \sim \N \Par*{\mu, \frac{1}{n_i}}
\]
to estimate $\mu$. This will give the following expected marginal costs
\[
  c_{i, \text{ind}}^e = \E_{S_1, S_2} c(s_{i, \text{ind}}) = a + b \E (s_{i,
  \text{ind}} - \mu)^2 = a + \frac{b}{n_i}.
\]
If the firms collaborate, they uses both samples to calculate average
\[
  s_\text{shared} = \frac{1}{n_1 + n_2} \sum_{i, j} X^i_j \implies s_\text{shared}
  \sim \N \Par*{\mu, \frac{1}{n_1 + n_2}}
\]
and the expected marginal costs will be
\[
  c_\text{shared}^e = \E_{S_1, S_2} c(s_\text{shared}) = a + \frac{b}{n_1 + n_2}.
\]

To solve the data-sharing problem, we need to describe the competition between
firms. For simplicity we will only consider the case of Cournot competition.

As in \cref{sec:supply}, the firms are interested in increasing their
expected profits
\[
  \varPi_i^e = \E_{S_1, S_2} (p_i q_i - c_i q_i).
\]
Thus, the firms use the same quantities as in \cref{sec:supply}. In the case
of non-collaboration, the expected profits are
\[
  \varPi_{i, \text{ind}}^e = \Par*{\frac{2 - \gamma + \gamma c_{-i}^e - 2
  c_i^e}{4 - \gamma^2}}^2.
\]
Similarly, in the case of collaboration,
\[
  \varPi^e_\text{shared} = \Par*{\frac{(2 - \gamma) (1 -
  c^e_{\text{shared}})}{4 - \gamma^2}}^2.
\]
The collaboration criterion is the same as in \cref{sec:supply}.

\subsection{Oil drilling}

Assume that an oil company manager needs to optimize the costs of oil rig
construction. Construction of an oil rig happens in two steps. First,
geologists look at a new place for a rig and produce a noisy signal of oil
presence $X$. Second, the company might try to build a rig. However, if there
is no oil, the company will lose money. Thus, sometimes it is more profitable
to repeat the geological expedition.

The noisy signal of geologists predicts the presence of oil in the following
manner
\[
  \Pr \Par*{\text{success} \given X} = \Pr \Par*{X + \epsilon \ge 0 \given X},
  \ \epsilon \sim \N(0, 1).
\]
Also, the manager knows that
\[
  X \sim \N(0, \sigma^2).
\]
However, the manager does not know $\sigma^2$ and needs to estimate it from the
previous observations. The new expedition costs $a$, and drilling a new rig
costs $b$.

The natural strategy is to repeat expeditions until the geologist find $X \ge
r$ and then try to build a rig. Let $p$ be the probability of finding a good
spot and $q$ be the probability of success in a good spot
\begin{align*}
  p \defeq& \Pr(X \ge r) = \Phi \Par*{\frac{-r}{\sigma}},\\
  q \defeq& \Pr \Par*{X + \epsilon \ge 0 \given X \ge r} = \frac{1}{p}
  \int_r^\infty \int_{-x}^\infty \frac{\exp \Par*{\frac{-x^2}{2
  \sigma^2}}}{\sqrt{2 \pi \sigma^2}} \frac{\exp \Par*{\frac{-y^2}{2}}}{\sqrt{2
  \pi}} \md y \md x = \frac{M(r, \sigma^2)}{p},
\end{align*}
where we denote the last integral by $M$. So, the expected cost is
\[
  f(r) = \frac{\frac{a}{p} + b}{q} = \frac{a + \Phi \Par*{\frac{-r}{\sigma}}
  b}{M(r, \sigma^2)}.
\]

Assume that the company estimates $\est{\sigma}^2$ of $\sigma^2$ and chooses
$r$ according to that estimate. The loss will approximately be
\begin{align*}
  f(r) - f(r^*) &\approx f'(r^*) (r - r^*) + \frac{f''(r^*)}{2} (r - r^*)^2 =
  \frac{f''(r^*)}{2} (r - r^*)^2 \\
  &\approx \frac{f''(r^*)}{2} r'(\sigma^2)^2 (\est{\sigma}^2 - \sigma^2)^2.
\end{align*}

Now, if the company uses MLE to estimate $\sigma^2$, then the costs will
asymptotically look like
\[
  f(r) \approx A + \frac{B}{n} \chi^2(1),
\]
where $n$ is the number of observations.

After that the process is the same as in previous subsection.

\section{Extensions of the results}
\label{s:robustness}

In this section, we investigate how different changes in our setting affect the
results. For simplicity, we only study the case of perfect substitutes ($\gamma
= 1$) and Cournot competition if not told otherwise.

\subsection{Different cost functions}
\label{s:diff-costs}

In this subsection, we assume that the cost functions of the firms have the
form
\[
  C_i(q) = c_i^e q + \frac{k}{2} q^2.
\]
The profit maximization problem of the firm now has the form
\[
  \max_{q_i} (1 - q_i - q_{-i}) q_i - c_i^e q_i - \frac{k}{2} q_i^2.
\]
Therefore, the best responses are
\[
  q_i^*(q_{-i}) = \frac{1 - q_{-i} - c_i^e}{2 + k}
\]
and equilibrium quantities are
\[
  q_i^* = \frac{1 + k + c_{-i}^e - (2 + k) c_i^e}{(1 + k) (3 + k)}.
\]
These equations imply the following expected profits
\[
  \varPi_i^e = \frac{2 + k}{2} (q_i^*)^2
\]
and result in the following collaboration criterion
\[
  \forall i \ (2 + k) (n_i^{-1} - (n_1 + n_2)^{-1}) > n_{-i}^{-1} - n_i^{-1}.
\]
As we can see, the added non-linearity in the cost does not qualitatively
change the conclusions of \cref{t:main}.

\subsection{Multiplicative substitution}
\label{s:diff-demand}

In the \cref{sec:setting}, we assume that goods produced by the firms can be
additively substituted by the outside goods (i.e., the utility function
function is quasilinear). However, in general, the firms might operate in
different markets with different substitution patterns.

To investigate, this question we assume the following form of utility function
\[
  u(q_0, q_1, q_2) = \Par*{q_1 + q_2 - \frac{q_1^2 + 2 q_1 q_2 +
  q_2^2}{2}}^\alpha q_0^{1 - \alpha}.
\]
($p_1 = p_2 = p$ because the goods $G_1$ and $G_2$ are perfect substitutes.)
This utility prescribes the consumer to spend the share $\alpha$ of income on
goods $G_1$ and $G_2$ ($B$ is small enough, so that $q_1$ and $q_2$ are much
smaller than 1). So, we have the following consumer problem
\[
  \max_{q_1, q_2} q_1 + q_2 - \frac{q_1^2 + 2 q_1 q_2 + q_2^2}{2} \text{ s.t. }
  p (q_1 + q_2) \le \alpha B.
\]
So, the demand is
\[
  p = \frac{\alpha B}{q_1 + q_2}.
\]

Expected profits maximization problems are
\[
  \max_{q_i} \frac{\alpha B q_i}{q_1 + q_2} - c^e_i q_i.
\]
The first order conditions give
\[
  \frac{\alpha B q_{-i}^*}{(q_1^* + q_2^*)^2} = c^e_i.
\]
So, equilibrium quantities are
\[
  q_i^* = \frac{\alpha B q_1^* q_2^*}{c_i (q_1^* + q_2^*)^2} = \frac{\alpha B
  c^e_2 c^e_1}{c^e_i (c^e_2 + c^e_1)^2}.
\]
And equilibrium profits are
\[
  \varPi_i^e = \frac{\alpha B (c^e_{-i})^2}{(c^e_1 + c^e_2)^2} = \frac{\alpha
  B}{\Par*{1 + \frac{c^e_i}{c^e_{-i}}}^2}.
\]
This formula imply the following collaboration criterion
\[
  \forall i \ 1 + \frac{c^e_{i, \text{ind}}}{c^e_{-i, \text{ind}}} < 1 +
  \frac{c^e_\text{share}}{c^e_\text{share}} = 2.
\]
As can be seen, the collaboration is always not profitable for the firm with
more data.

\subsection{Asymmetric costs}
\label{s:asym-costs}

In this subsection, we assume that firms have asymmetric cost functions
\[
  c_i^e(n) = a_i + \frac{b_i}{n^{\beta_i}},
\]
where $c_i^e(n)$ is the expected marginal cost of the company if it use $n$
points for training. \cref{sec:supply} gives the following the expected
profits
\[
  \varPi_i^e = \Par*{\frac{2 - \gamma - 2 c_i^e + \gamma c_{-i}^e}{4 -
  \gamma^2}}^2.
\]
As in \cref{sec:supply}, the companies will collaborate only if both of
them profit from collaboration
\[
  \forall i \ \varPi^e_\text{shared} > \varPi^e_{i, \text{ind}} \iff 2 b_i
  (n_i^{-\beta_i} - n_\text{shared}^{-\beta_i}) > \gamma b_{-i}
  (n_{-i}^{-\beta_{-i}} - n_\text{shared}^{-\beta_{-i}}).
\]
Denote $f_i(n_1, n_2) = 2 b_i (n_i^{-\beta_i} - n_\text{shared}^{-\beta_i}) -
\gamma b_{-i} (n_{-i}^{-\beta_{-i}} - n_\text{shared}^{-\beta_{-i}})$. This
function can be interpreted as firm $F_i$ reediness to collaborate: the higher
this value is the more firm $F_i$ wants to collaborate with its competitor
$F_{-i}$.

\begin{theorem}
  Assume the setting described above. The function $f_i(n_1, n_2)$ has the
  following properties:
  \begin{enumerate}
    \item $f_i(n_1, n_2)$ is decreasing in $\gamma$
    \item $f_i(n_1, n_2)$ is decreasing in $\beta_i$ and increasing in
      $\beta_{-i}$ if $\ \forall i \ \beta_i \ln n_i > 1$.
    \item $f_i(n_1, n_2)$ is increasing in $b_i$ and decreasing in $b_{-i}$.
  \end{enumerate}
\end{theorem}

\begin{proof}
  The first property is evident
  \[
    \pdv{f_i}{\gamma} = - b_{-i} (n_i^{-\beta_{-i}} -
    n_\text{shared}^{-\beta_{-i}}) < 0.
  \]

  We can prove the last property similarly
  \begin{align*}
    \pdv{f_i}{b_i} =  2 (n_i^{-\beta_i} - n_\text{shared}^{-\beta_i}) > 0,\\
    \pdv{f_i}{b_{-i}} = -\gamma (n_i^{-\beta_{-i}} -
    n_\text{shared}^{-\beta_{-i}}) < 0.
  \end{align*}

  The first part of the second property follows from the following
  \[
    \pdv{f_i}{\beta_i} = 2 b_i \Par*{\frac{\ln n_i}{n_\text{shared}^{\beta_i}}
    - \frac{\ln n_i}{n_i^{\beta_i}}}.
  \]
  Notice that the function $g(x) = \frac{\ln x}{x^{\beta_i}}$ is decreasing
  in $x$ on $x > e^{1 / \beta_i}$ because
  \[
    \pdv{g}{x} = \frac{1 - \beta_i \ln x}{x^{\beta_i + 1}} < 0.
  \]
  Therefore, $\pdv{f_i}{\beta_i} < 0$: $f_i$ is decreasing in $\beta_i$.
  The second part can be proved similarly.
\end{proof}

The first property is similar to the first property in \cref{t:main}. The firms
with more similar goods have less incentives to collaborate.

The last two properties show that the firm wants to collaborate with its
competitor when the firm's machine learning model is bad ($\beta_i$ is low or
$b_i$ is high) or the competitor's model is good ($\beta_{-i}$ is high or
$b_{-i}$ is low).

\subsection{Heterogeneity}
\label{s:heterogeneity}

In this section, we solve the same model as in \cref{sec:example}, but
allowing for heterogeneity between firms. Each firm will need to optimize the
MSE costs
\[
  c_i(s) = a - b + b \E((s - X_i)^2) = a + b (s - \mu_i)^2, \ X_i \sim
  \N(\mu_i, 1).
\]
On the contrary to the \cref{sec:example}, the means of the noise will be
different among firms and will be drawn at the start of the game from the
distribution
\[
  \mu_i \sim \N(\mu, \sigma^2).
\]
We assume that $\mu$ is unknown and $\sigma^2$ is known.

When firms do not collaborate, the expected costs are the same as in
\cref{sec:example}. When firms collaborate, MLE for the means are
\[
  s_{i, \text{shared}} = \frac{2 \sigma^2 n_1 n_2 Y_i + n_i Y_i + n_{-i} Y_{-i}}{2
  \sigma^2 n_1 n_2 + n_1 + n_2},
\]
where $n_i$ is the number of data points of firm $F_i$, $Y_i$ is the average of
noise. Then expected marginal costs will be
\[
  c^e_{i, \text{shared}} = a + b \E_{S_1, S_2} ((s_{i, \text{shared}} - \mu_i)^2).
\]
By substitution,
\begin{multline*}
  \E((s_{i, \text{shared}} - \mu_i)^2) =\\
  \E \Par*{\frac{n_{-i}^2 (\mu_1 - \mu_2)^2}{(2 \sigma^2 n_1 n_2 + n_1 +
  n_2)^2} + \frac{(2 \sigma^2 n_{-i} + 1)^2 n_i}{(2 \sigma^2 n_1 n_2 + n_1 +
  n_2)^2} + \frac{n_{-i}}{(2 \sigma^2 n_1 n_2 + n_1 + n_2)^2}}\\
  = \frac{2 \sigma^2 n_{-i} + 1}{2 \sigma^2 n_1 n_2 + n_1 + n_2}.
\end{multline*}
Therefore, $c_{i, \text{shared}}^e$ is increasing in $\sigma^2$: the more
different the firms are, the less collaboration gives. Two corner cases
give
\begin{align*}
  \sigma^2 = 0 &\implies c_{i, \text{shared}}^e = a + \frac{b}{n_1 + n_2},\\
  \sigma^2 \to \infty &\implies c_{i, \text{shared}}^e \to a + \frac{b}{n_i}.
\end{align*}

However, notice that collaboration criteria in the Cournot case does not depend
on $\sigma^2$ and hence does not change
\begin{multline*}
  \Pi_{i, \text{share}}^e - \Pi_{i, \text{ind}}^e \ge 0 \iff
  2 - \gamma - 2 c_{i, \text{share}}^e + \gamma c_{-i, \text{share}}^e \ge 2 -
  \gamma - 2 c_{i, \text{ind}}^e + \gamma c_{-i, \text{ind}}^e\\
  \iff - \frac{2 (2 \sigma^2 n_{-i} + 1)}{2 \sigma^2 n_1 n_2 + n_1 + n_2} +
  \frac{\gamma (2 \sigma^2 n_i + 1)}{2 \sigma^2 n_1 n_2 + n_1 + n_2} \ge
  -\frac{2}{n_i} + \frac{\gamma}{n_{-i}}
  \iff 0 \ge -2 \frac{n_{-i}}{n_i} + \gamma \frac{n_i}{n_{-i}}.
\end{multline*}

\section{Other models of coalition formation}
\label{sec:other-coals}

In this section, we present models of coalition formation different form the
model presented in \cref{sec:many_firms}.

\subsection{Cooperative data sharing}
\label{sec:coop_game}

We use an solution concept motivated by the $\alpha$-core \citep{v44}, which is
standard in cooperative game theory. The $\alpha$-core consists of partitions
that are incentive-compatible, ensuring that no subset of firms wants to
deviate. Incentive compatibility indicates that these partitions more stable
and likely to occur.

First, we introduce a modification of $\alpha$-core, tailored to our setup.
This deviation is necessary due to two reasons, non-transferable utility and
the presence of externalities (dependencies of utilities on the actions of all
participants), which arise in our setting because all firms compete with each
other on the market.

Let $N = \{F_1, \dots, F_m\}$ be the set of all firms. A partition
of $N$ is a set of subsets $\{S_1, \dots, S_k\} : S_i \subseteq N$, such that
$\forall i, j, i \neq j: \ S_i \cap S_j = \emptyset$ and $\cup_i S_i = N$. We
denote the set of all partitions of $N$ as $P(N)$.
\begin{definition}[$\alpha$-Core]
  A partition $P$ belongs to the $\alpha$-core of the game if and only if
  \[
    \nexists S \subseteq N \colon \forall F \in S, Q \in P(N \setminus S) \quad
    \varPi^e_F(P) < \varPi^e_F(Q \cup \{S\}),
  \]
  where $\varPi^e_F(P)$ is expected profit of the firm $F$ if market coalition
  structure is $P$.
\end{definition}

Intuitively, a partition $P$ belongs to the $\alpha$-core if no set of
companies wants to deviate from this partition. Here we say that a subset of
companies wants to deviate, if any company in this subset will increase its
profits by joining this new coalition, regardless of how the remaining
companies split into groups.

In the next theorem we identify one coalition structure that always belongs to
the $\alpha$-core of the data-sharing game, demonstrating that the
not-emptiness of the core.
\begin{theorem}
  \label{t:collab}
  W.l.o.g. assume that $n_1 \ge n_2 \ge \dots \ge n_m$, where $n_i$ is the
  number of data points of firm $F_i$. Consider partitions $P_i = \{A_i, N
  \setminus A_i\}$, where $A_i = \{F_1, F_2, \dots, F_i\}$. Let $i^* =
  \argmax_i \varPi^e_{F_1}(P_i)$. Then $P_{i^*}$ belongs to $\alpha$-core.
\end{theorem}

Before we start to prove this theorem, we formulate the following lemma.
\begin{lemma}
  \label{l:collab}
  Let $Q \subset N$ and $F \in Q$. Then
  \[
    \varPi^e_F(\{Q, N \setminus Q\}) =
    \frac{(2 - \gamma - (2 + \gamma (m - 1 - \abs{Q}) c^e_Q + \gamma (m -
    \abs{Q}) c^e_{N \setminus Q})^2}{(2 - \gamma)^2 (2 + (m - 1) \gamma)^2},
  \]
  where $c^e_X = a + b / n_X^\beta$, expected marginal cost of coalition $X$,
  and $n_X = \sum_{F \in X} n_F$, the number of data points of coalition $X$.
\end{lemma}

\begin{proof}
  The results of \cref{sec:supply} give the following expected profits
  \[
    \varPi_i^e = \frac{(2 - \gamma - (2 + \gamma (m - 2)) c_i^e + \gamma \sum_{j
    \neq i} c_j^e)^2}{(2 - \gamma)^2 (2 + (m - 1) \gamma)^2}.
  \]
  In our case,
  \[
    c_i =
    \begin{cases}
      c_Q, & F_i \in Q,\\
      c_{N \setminus Q}, & F_i \notin Q,
    \end{cases}
  \]
  By substituting $c_i^e$ into the profit equation for $F \in Q$, we get
  \begin{align*}
    \varPi^e_F(\{Q, N \setminus Q\}) &=
    \frac{(2 - \gamma - (2 + \gamma (m - 2)) c^e_Q + \gamma (\abs{Q} - 1) c^e_Q
    + \gamma (m - \abs{Q}) c^e_{N \setminus Q})^2}{(2 - \gamma)^2 (2 + (m - 1)
    \gamma)^2} \\
    &= \frac{(2 - \gamma - (2 + \gamma (m - 1 - \abs{Q}) c^e_Q + \gamma (m -
    \abs{Q}) c^e_{N \setminus Q})^2}{(2 - \gamma)^2 (2 + (m - 1) \gamma)^2}.
  \end{align*}
\end{proof}

\begin{corollary}
  \label{c:collab}
  Let $Q, Q' \subset N$, $F \in Q$, and $F' \in Q$. Then
  \begin{align*}
    \varPi^e_F(\{Q, N \setminus Q\}) &\ge \varPi^e_{F'}(\{Q', N \setminus Q'\}) \iff\\
    \frac{2 + \gamma (m - 1 - \abs{Q'})}{n_{Q'}^\beta} - \frac{\gamma (m -
    \abs{Q'})}{n_{N \setminus Q'}^\beta} &\ge
    \frac{2 + \gamma (m - 1 - \abs{Q})}{n_Q^\beta} - \frac{\gamma (m -
    \abs{Q})}{n_{N \setminus Q}^\beta}.
  \end{align*}
\end{corollary}

Now, we are ready to prove \cref{t:collab}.
\begin{proof}
  We prove the theorem by contradiction. Suppose that there exists a coalition
  $Q$ for which deviation is profitable.

  First, we consider the case $A_{i^*}
  \cap Q \neq \emptyset$. Let $F \in A_{i^*} \cap Q$. Notice that
  \[
    \varPi^e_F(P_{i^*}) = \varPi^e_{F_1}(P_{i^*}) \ge \varPi^e_{F_1}(P_{\abs{Q}}) \ge
    \varPi^e_F(\{Q, N \setminus Q\}).
  \]
  The first equality follows from \cref{l:collab}. The first inequality
  follow from the definition of $i^*$. The last inequality follows from
  \cref{c:collab} and observation $n_{A_{\abs{Q}}} \ge n_Q$. So, we get a
  contradiction. Thus, $A_{i^*} \cap Q = \emptyset$.

  Let $F \in Q$. In this case,
  \[
    \varPi^e_F(P_{i^*}) \ge \varPi^e_F(\{Q, N \setminus Q\}).
  \]
  This inequality follows from \cref{c:collab}, the fact that $n_{N
  \setminus A_{i^*}} \ge n_Q$, and observation $\abs{N \setminus A_{i^*}} \ge
  \abs{Q}$. This contradiction proves the theorem.
\end{proof}

In the partition described in the theorem, the firms form two coalitions: one
contains the $i^*$ companies with the largest datasets and the other one
contains all other firms and leads to largest profits for the company with the
largest amount of data. However, it may not be the only one in the
core since the $\alpha$-core is one of the most permissive cores in the
cooperative game literature: a set of companies will deviate only if it is
better-off irrespective of the actions of others, sometimes resulting
non-economically plausible partitions.

The following example illustrates this flaw. Assume that there are three firms
in the market and $n_1 \gg n_2 \gg n_3$. In this case, the $\alpha$-core
consists of two equilibria: $\{\{F_1, F_2\}, \{F_3\}\}$ and $\{\{F_1\},
\{F_2\}, \{F_3\}\}$, but the first equilibrium is less economically plausible.
Indeed, for the first firm, the second equilibrium is more profitable than the
first. However, this firm will not deviate because of being afraid that the
second firm might collaborate with the third. Nevertheless, this fear is
ungrounded since the second firm expects less profit from the collaboration
with the third, than from staying alone.

\subsection{Universal data-sharing treaty}

The first model consider a following situation. Imagine that companies might
sign some treaty that will their data accessible for use for other companies.
To make companies comply to the outcome, the decision is made by consensus. In
terms of coalition formation game, each company either agree or disagree to
participate in grand coalition. If any company disagrees, the grand coalition
is not formed and companies act like singletons.

W.l.o.g., assume that $n_1 \ge n_2 \ge \dots \ge n_m$. The grand coalition will
be Nash equilibrium in this game only if all firms will be better off in the
grand coalition
\[
  \forall F \ \varPi^e_F(\{\{F_1, F_2, \dots, F_m\}\}) > \varPi^e_F(\{\{F_1\},
  \{F_2\}, \dots, \{F_m\}\}).
\]
Since grand coalition profit is the same for all firms and singleton profit is
biggest for the firm with the most amount of data, this system of inequalities
is equivalent to the inequality for the first firm
\[
  \varPi^e_{F_1}(\{\{F_1, F_2, \dots, F_m\}\}) > \varPi^e_{F_1}(\{\{F_1\},
  \{F_2\}, \dots, \{F_m\}\}).
\]
By substitution, this inequality takes the form
\[
  \frac{\gamma m - m - 1}{n^\beta} > -\frac{m + 1}{n_1^\beta} +
  \sum_{j = 1}^m \frac{\gamma}{n_j^\beta},
\]
where $n = \sum_j n_j$.

\subsection{Data-sharing treaty}

The second model is similar to the first, but here the members of the treaty
share data only between themselves. Thus, not everybody needs to agree to build
a coalition and resulting structure will be one coalition with several members
and singleton coalitions. Formally, we assume the following game. Each firm has
two actions: Y and N. All firms that answer Y form a coalition between
themselves and all firms that answer N form singleton coalitions. We are
interested in Nash equilibria of this game.

W.l.o.g., assume that $n_1 \ge n_2 \ge \dots \ge n_m$. The following lemma is
useful for the description of all equilibria of this game
\begin{lemma}
  \label{l:treaty_characterization}
  Let $S$ be a set of firm that answer Y in equilibrium and $i^* = \min \Bc*{i
  \given F_i \in S}$. Then $\forall j \ge i^* \ F_j \in S$.
\end{lemma}

\begin{proof}
  To show this it is sufficient to show that a single firm $F$ always want to
  join a coalition $S$ that has more data than the firm. This is equivalent to
  the following inequality
  \begin{align*}
    \varPi^e_F(\{S \cup \{F\}\} \cup \Bc*{\{G\} \given G \notin S \cup \{F\}}) &>
    \varPi^e_F(\{S\} \cup \Bc*{\{G\} \given G \notin S}) \\
    \iff \frac{\gamma (\abs{S} + 1) - m - 1}{(n_S + n_F)^\beta} &>
    \frac{\gamma \abs{S}}{n_S^\beta} - \frac{m + 1 - \gamma}{n_F^\beta},
  \end{align*}
  where $n_S = \sum_{G \in S} n_G$. Denote $y \defeq \frac{n_F}{n_S}$. The last
  inequality is equivalent to
  \[
    (m + 1 - \gamma) \Par*{1 - \frac{y^\beta}{(1 + y)^\beta}} > \gamma \abs{S}
    \Par*{y^\beta - \frac{y^\beta}{(1 + y)^\beta}}.
  \]
  Since $\gamma < 1$ and $\abs{S} < m$, we have $m + 1 - \gamma > \gamma
  \abs{S}$. Furthermore, $1 - \frac{y^\beta}{(1 + y)^\beta} \ge y^\beta -
  \frac{y^\beta}{(1 + y)^\beta} > 0$ because $1 \ge y > 0$. Therefore,
  the inequality above holds.
\end{proof}

\cref{l:treaty_characterization} greatly constraints possible equilibria of the
game. Indeed, the equilibrium set of companies that say Y may only have the
form $S_i = \Bc*{F_j \given j \ge i}$. To check whether $S_i$ appears in the
equilibrium, we need to check the following system of inequalities
\begin{align*}
  \forall j < i \ &\varPi^e_{F_j}(\{S_i\} \cup \Bc{\{G\} \given G \notin S_i}) \ge
  \varPi^e_{F_j}(\{S_i \cup \{F_j\}\} \cup \Bc{\{G\} \given G \notin S_i \cup
  \{F_j\}}),\\
  \forall j \ge i \ &\varPi^e_{F_j}(\{S_i\} \cup \Bc{\{G\} \given G \notin S_i})
  \le \varPi^e_{F_j}(\{S_i \cup \{F_j\}\} \cup \Bc{\{G\} \given G \notin S_i \cup
  \{F_j\}}).
\end{align*}
These inequalities are closest for the firms around threshold $i$. Thus, this
system is equivalent to two inequalities
\begin{align*}
  \varPi^e_{F_{i-1}}(\{S_i\} \cup \Bc{\{G\} \given G \notin S_i}) &\ge
  \varPi^e_{F_{i-i}}(\{S_i \cup \{F_{i-1}\}\} \cup \Bc{\{G\} \given G \notin S_i
  \cup \{F_{i-1}\}}),\\
  \varPi^e_{F_i}(\{S_i\} \cup \Bc{\{G\} \given G \notin S_i}) &\le
  \varPi^e_{F_i}(\{S_i \cup \{F_i\}\} \cup \Bc{\{G\} \given G \notin S_i \cup
  \{F_i\}}).
\end{align*}

\section{Welfare analysis}
\label{sec:welfare}

In this section, we consider how collaboration affects the cumulative utility
of firms and consumers. We start with the definition of welfare.
\begin{definition}
  The sum of consumer surplus and firms surplus $W$ is called welfare. In our
  term this quantity is equal to the sum of consumers utility and firms profits
  \[
    W \defeq u + \sum_{i = 1}^m \varPi_i.
  \]
  The definition of welfare in our case is simpler than in general case because
  consumers have quasilinear preferences.
\end{definition}

Welfare is an important quantity in policy analysis. Welfare reflects the
cumulative benefits of all market participants. Thus, if redistribution between
people is possible (e.g., through government transfers), increase in welfare
means increase in utility for all people given the right redistribution scheme.

Now, we want to evaluate how data-sharing in duopoly case affects the expected
welfare. This will help us to understand in which cases the firms should be
incentivized to collaborate with each other. By direct substitution, we get
\[
  W^e = q_1^* + q_2^* - \frac{(q_1^*)^2 + (q_2^*)^2 + 2 \gamma q_1^* q_2^*}{2} +
  B - c_1^e q_1^* - c_2^e q_2^*.
\]

We show that collaboration increases welfare in the Cournot case if the
marginal costs of the firms are close enough: $\abs{c_1^e - c_2^e} \le
\frac{\gamma}{6}$. To accomplish that, we will first demonstrate that
$\pdv{W}{c_i^e} < 0$ if $c_i^e > c_{-i}^e$ and then show that $W$ decreases
along the line $c_1^e = c_2^e = c$. These properties will show that
collaboration increases welfare since after collaboration the costs of the
companies equalizes and become smaller.

The proof of the first part follows from the derivations below. W.l.o.g. assume
that $n_1 \ge n_2$, then $c_1^e \le c_2^e$. Direct computations show
\[
  \pdv{W^e}{c_2^e} = \pdv{q_1^*}{c_2^e} (1 - c_1^e) + \pdv{q_2^*}{c_2^e} (1 -
  c_2^e) - \pdv{q_1^*}{c_2^e} (q_1^* + \gamma q_2^*) - \pdv{q_2^*}{c_2^e}
  (q_2^* + \gamma q_1^*) - q_2^*.
\]
The results of \cref{sec:supply} give
\[
  q_i^* = \frac{2 - \gamma - 2 c_i^e + \gamma c_{-i}^e}{4 - \gamma^2}.
\]
Substituting the expressions above, we have
\begin{multline*}
  \pdv{W^e}{c_2^e} = \frac{(12 - \gamma^2) c_2^e - (8 \gamma - \gamma^3) c_1^e
  - (12 - 8 \gamma - \gamma^2 + \gamma^3)}{(4 - \gamma^2)^2}\\
  < \frac{(12 - \gamma^2) c_2^e - (8 \gamma - \gamma^3) \Par*{c_2^e -
  \frac{\gamma}{6}} - (12 - 8 \gamma - \gamma^2 + \gamma^3)}{(4 -
  \gamma^2)^2}\\
  = \frac{- (12 - 8 \gamma - \gamma^2 + \gamma^3) (1 - c_2^e) +
  \frac{\gamma}{6} (8 \gamma - \gamma^3)}{(4 -
  \gamma^2)^2}.
\end{multline*}
The restriction $c_2 \le 1 - \frac{\gamma}{2}$ gives
\[
  \pdv{W^e}{c_2^e} \le \frac{- (12 - 8 \gamma - \gamma^2 + \gamma^3)
  \frac{\gamma}{2} + \frac{\gamma}{6} (8 \gamma - \gamma^3)}{(4 - \gamma^2)^2}
\]
This inequality is equivalent to the following inequality
\[
  36 - 24 \gamma - 3 \gamma^2 + 3 \gamma^3 - 8 \gamma + \gamma^3 > 0 \iff (36 -
  3 \gamma^2) (1 - \gamma) + \gamma^3 > 0.
\]
The last inequality holds since $\gamma < 1$.

The proof of the second part comes from the computations below. Assume $c_1^e =
c_2^e = c$, we want to show that $W$ decreases in $c$. We have
\begin{align*}
  W^e &= \frac{(12 - 8 \gamma - \gamma^2 + \gamma^3) c^2 - (24 - 16 \gamma - 2
  \gamma^2 + 2 \gamma^3) c + 12 - 8 \gamma - \gamma^2 + \gamma^3}{(4 -
  \gamma^2)^2} + B\\
  &= \frac{(12 - 8 \gamma - \gamma^2 + \gamma^3) (1 - c)^2}{(4 - \gamma^2)^2} +
  B.
\end{align*}
Clearly, this function is decreasing in $c$ for $c < 1$.

\section{Additional experiments for
\texorpdfstring{\cref{sec:many_firms}}{Section \ref{sec:many_firms}}}
\label{sec:add-coop}

\begin{figure}[ht]
  \centering
  \vskip -0.2in
  \input{dep_m3.pgf}
  \caption{The dependence of the average coalition size on $\gamma$ and $\beta$
  for in synthetic experiments. The $y$-axes report the mean of the average
  size of the coalitions in the equilibrium partition, where mean is taken over
  $10000$ Monte Carlo simulations of the game. The firms' dataset sizes are
  sampled from clipped (at $1$) distribution $P$. There are $3$ firms in the
  market engaging in the Cournot competition. The default values of $\gamma$
  and $\beta$ are equal to $0.8$ and $0.9$, respectively.}
\end{figure}

\begin{figure}[ht]
  \centering
  \vskip -0.2in
  \input{dep_m4.pgf}
  \caption{The dependence of the average coalition size on $\gamma$ and $\beta$
  for in synthetic experiments. The $y$-axes report the mean of the average
  size of the coalitions in the equilibrium partition, where mean is taken over
  $10000$ Monte Carlo simulations of the game. The firms' dataset sizes are
  sampled from clipped (at $1$) distribution $P$. There are $4$ firms in the
  market engaging in the Cournot competition. The default values of $\gamma$
  and $\beta$ are equal to $0.8$ and $0.9$, respectively.}
\end{figure}

\begin{figure}[ht]
  \centering
  \vskip -0.2in
  \input{dep_m5.pgf}
  \caption{The dependence of the average coalition size on $\gamma$ and $\beta$
  for in synthetic experiments. The $y$-axes report the mean of the average
  size of the coalitions in the equilibrium partition, where mean is taken over
  $10000$ Monte Carlo simulations of the game. The firms' dataset sizes are
  sampled from clipped (at $1$) distribution $P$. There are $5$ firms in the
  market engaging in the Cournot competition. The default values of $\gamma$
  and $\beta$ are equal to $0.8$ and $0.9$, respectively.}
\end{figure}

\end{document}